\documentclass[journal]{IEEEtran}

\usepackage{cite}

\ifCLASSINFOpdf
  \usepackage[pdftex]{graphicx}
  \graphicspath{{./figures/}}
 
\else
 
\fi

\usepackage{amsmath}
\usepackage{amssymb}
\usepackage{amsthm}
\usepackage{algorithmic}
\usepackage{algorithm}
\usepackage{dsfont} 
\usepackage{url}
\usepackage{booktabs}
\usepackage{picinpar}
\usepackage{multirow}
\usepackage{color}
\usepackage{graphicx}
\usepackage{subfigure}

\newtheorem{mydef}{Deﬁnition}[section]
\newtheorem{mytheo}{Theorem}[section]
\newtheorem{mylemma}{Lemma}
\allowdisplaybreaks  
\begin{document}

\title{Low-Multi-Rank High-Order Bayesian Robust Tensor Factorization}

\author{Jianan~Liu and~Chunguang~Li,~\IEEEmembership{Senior Member,~IEEE}
\thanks{This work was supported by the National Key Research and Development Program of China (Grant No. 2018YFE0125400), the National Natural Science Foundation of China (Grant No. U20A20158), and the National Program for Special Support of Eminent Professionals. }
\thanks{The authors are with the College of Information Science and Electronic Engineering, Zhejiang University, Hangzhou 310027, China. C. Li is also with the Ningbo Research Institute, Zhejiang University, Ningbo 315100, China (C. Li is the corresponding author, email: cgli@zju.edu.cn). }

}

%



\maketitle

\begin{abstract}
  The recently proposed tensor robust principal component analysis (TRPCA) methods based on tensor singular value decomposition (t-SVD) have achieved numerous successes in many fields. However, most of these methods are only applicable to third-order tensors, whereas the data obtained in practice are often of higher order, such as fourth-order color videos, fourth-order hyperspectral videos, and fifth-order light-field images. Additionally, in the t-SVD framework, the multi-rank of a tensor can describe more fine-grained low-rank structure in the tensor compared with the tubal rank. However, determining the multi-rank of a tensor is a much more difficult problem than determining the tubal rank. Moreover, most of the existing TRPCA methods do not explicitly model the noises except the sparse noise, which may compromise the accuracy of estimating the low-rank tensor. In this work, we propose a novel high-order TRPCA method, named as Low-Multi-rank High-order Bayesian Robust Tensor Factorization (LMH-BRTF), within the Bayesian framework. Specifically, we decompose the observed corrupted tensor into three parts, i.e., the low-rank component, the sparse component, and the noise component. By constructing a low-rank model for the low-rank component based on the order-$d$ t-SVD and introducing a proper prior for the model, LMH-BRTF can automatically determine the tensor multi-rank. Meanwhile, benefiting from the explicit modeling of both the sparse and noise components, the proposed method can leverage information from the noises more effectivly, leading to an improved performance of TRPCA. Then, an efficient variational inference algorithm is established for parameters estimation. Empirical studies on synthetic and real-world datasets demonstrate the effectiveness of the proposed method in terms of both qualitative and quantitative results.
\end{abstract}


%
\IEEEpeerreviewmaketitle

\section{Introduction}

\IEEEPARstart{T}{ensors}, or multi-dimensional arrays, are a natural form of representation for data acquired in practice. Compared wtih matrices, tensors can more naturally and accurately characterize the intrinsic structural information in multi-dimensional data, and thus have received more attention in recent years across many fields such as statistics\cite{1Statis}, signal processing\cite{2SignalProces}, computer vision\cite{3CV} and machine learning\cite{4MachinLearn}.

In practice, the collected tensor is inevitably corrupted by sensor noise, communication errors, etc. To uncover the underlying information from the corrupted tensor, a common way is to try to infer the underlying low-rank structure from the corrupted observation with the help of low-rank tensor decomposition, which is often referred to as Tensor Robust Principal Component Analysis (TRPCA). Given an observed corrupted tensor $\mathcal{Y}\in \mathbf{R}^{I_1 \times \cdots \times I_d}$, in literature TRPCA methods generally assume that $\mathcal{Y} = \mathcal{X}+\mathcal{S}$ and try to infer $\mathcal{X}$ from $\mathcal{Y}$, where $\mathcal{X}$ is a tensor with some low-rank structure and $\mathcal{S}$ is a sparse tensor.

TRPCA methods vary with different tensor decomposition frameworks. There are many different forms of decomposition for tensors. One classical form of low-rank tensor decomposition is the CP decomposition, and the corresponding rank is the CP rank\cite{5CPrank}. As accurate estimation of the CP rank is intractable and this problem is NP-hard\cite{CPrankNP1,CPrankNP2,CPrankNP3}, existing CP-based TRPCA methods resort to the probabilistic framework to estimate the CP factor tensors and the CP rank\cite{BRTF,CPdecomp2}. For example, by introducing proper priors into the model, Bayesian Robust Tensor Factorization (BRTF)\cite{BRTF} achieves effective low-CP-rank tensor estimation and automatic CP-rank determination within the Bayesian framework. Another classical form of low-rank tensor decomposition is the Tucker decomposition, and the corresponding rank is the Tucker rank\cite{Tuckerrank1,Tuckerrank2}. Many Tucker-based TRPCA methods\cite{SNN,Tuckerconvex} use the sum of nuclear norms (SNNs) as a convex surrogate of the Tucker rank. Although CP decomposition and Tucker decomposition are very classical forms of tensor decomposition, they generally lack enough ﬂexibility in characterizing the low-rank structures of complex tensors.

Recently, many novel low-rank tensor decomposition frameworks have been proposed, such as tensor train\cite{TT1,TT2}, tensor ring\cite{TR1,TR2}, tensor singular value decomposition (t-SVD)\cite{tSVD1,tSVD2}, etc. Compared with other forms of tensor decomposition, t-SVD has been demonstrated to be more effective in characterizing the low-rank structural information of tensors with ﬁxed orientation or spatial-shifting, e.g., color images, videos, and multi-channel audio sequences\cite{tSVD2,tsvdAdvan1,tsvdAdvan2,tsvdAdvan3,tsvdAdvan4}. Therefore, many recently proposed TRPCA methods are based on t-SVD\cite{TNN,BTRTF,tSVDmethod1,tSVDmethod2,tSVDmethod3}. For example, in \cite{TNN}, the authors deduce the tensor nuclear norm (TNN) based on t-SVD and then formulate a convex optimization problem for TRPCA. In \cite{BTRTF}, the authors construct a Bayesian generative model, BTRTF, based on t-SVD and develop a variational inference algorithm to estimate the model parameters for effective TRPCA. However, most of the existing t-SVD-based TRPCA methods are only applicable to third-order tensors, whereas the data obtained in practice are often of higher order, such as fourth-order color videos, fourth-order hyperspectral videos, and fifth-order light-field images. To solve this problem, a trivial way is to reshape the higher-order tensors into matrices or third-order tensors and then use the corresponding TRPCA methods. However, such a strategy inevitably breaks the correlations among different dimensions of the original tensor, and consequently degrades the TRPCA performance. In addition, in the t-SVD framework, two new tensor ranks, i.e., tubal rank and multi-rank are defined, where the tubal rank is the largest element of the multi-rank\cite{tSVD1}. Accurately determining the multi-rank of a tensor rather than just the tubal rank helps us to capture the more fine-grained low-rank structure in the tensor, which improves the performance of algorithm. However, multi-rank determination is a much more difficult problem than tubal rank determination.

There are also t-SVD-based methods for higher-order tensors in the literature, e.g. \cite{HTNN}. However, it is not applicable to TRPCA. In addition, this method is unable to automatically determine the multi-rank but just the tubal rank of tensors. Moreover, most of the existing TRPCA methods do not explicitly model the noises except the sparse noise\cite{SNN,TNN,tSVDmethod2,tSVDmethod3}, which may compromise the accuracy of estimating the low-rank tensor. In view of these problems, we expect a new TRPCA method, which can be applicable to higher-order tensors, can automatically determine the multi-rank, and explicitly models the noises except the sparse noise.

In this work, we proposed a novel high-order TRPCA method, named as Low-Multi-Rank High-Order Bayesian Robust Tensor Factorization (LMH-BRTF). Specifically, we assume that the observed corrupted tensor $\mathcal{Y}$ can be decomposed into a low-rank component $\mathcal{X}$, a sparse component $\mathcal{S}$ and a noise component $\mathcal{E}$, i.e., $\mathcal{Y} = \mathcal{X} + \mathcal{S} + \mathcal{E}$. Based on this assumption, we construct a low-rank model for $\mathcal{X}$ based on the order-$d$ t-SVD\cite{HTNN,orderDtSVD1,orderDtSVD2,orderDtSVD3} so as to handle higher-order tensors, and introduce a proper prior for the model. Meanwhile, we explicitly model both the sparse component $\mathcal{S}$ and the noise component $\mathcal{E}$ using probabilistic model. Furthermore, we propose a Bayesian framework that can automatically determine the multi-rank of the tensor and leverage the information from the noises, thus the performance of TRPCA is improved. 

The contribution of this work is three-fold:
\begin{enumerate}
  \item We propose a high-order Bayesian TRPCA framework based on order-$d$ t-SVD, which can effectively infer the low-multi-rank tensor from the high-order corrupted observation tensor, and can automatically determine the multi-rank of the high-order tensor.
  \item Different from BTRTF\cite{BTRTF} (only applicable to the third-order tensors) where the model is specified in the original domain, since order-$d$ t-SVD is based on the invertible linear transforms and most of its definitions and operations are in the transform domain, we specify the main part of our model in the transform domain.
  \item As exact inference of the proposed Bayesian model is analytically intractable, we propose a variational inference algorithm that crosses the original domain and the transform domain for efﬁcient parameters estimation.
  
\end{enumerate}

\section{Notations and Preliminaries}
This section introduces notations used in this paper and the order-$d$ t-SVD framework.

\begin{table}[]
  \centering
  \caption{Summary of Notations}
  \label{tab:pre1}
  \resizebox{\columnwidth}{!}{%
  \begin{tabular}{ll}
  \hline\hline
  Notation                                                                  & Description                                   \\ \hline
  $\mathcal{X}\in   \mathbf{R}^{I_1 \times I_2 \times    \cdots \times I_d}$ & the order-$d$ tensor                           \\
  $\mathcal{X}^{\dagger}   \in \mathbf{R}^{I_2 \times I_1 \times \cdots \times  I_d}$ & the conjugate transpose of $\mathcal{X}$                   \\
  $\left\|\mathcal{X}\right\|_{F}$                                           & the Frobenius norm of $\mathcal{X}$            \\
  $\mathcal{X}_{i_1,   \cdots, i_d}$                                         & $(i_1, \cdots, i_d)$-th entry of $\mathcal{X}$ \\
  $\mathcal{X}_{:,i_2,\cdots,   i_d} \in \mathbf{R}^{I_1 \times 1}$          & the mode-1 fiber of $\mathcal{X}$              \\
  $\mathcal{X}_{i_1,:,\cdots,   i_d} \in \mathbf{R}^{1 \times I_2}$          & the mode-2 fiber of $\mathcal{X}$              \\
  $\mathbf{X}^{(i_3,   \cdots,i_d)} \in \mathbf{R}^{I_1 \times I_2}$                      & the slice along mode-1 and mode-2 of $\mathcal{X}$         \\
  $\mathbf{X}^{j}   \in \mathbf{R}^{I_{1} \times I_{2}}$                              & the $j$-th slice along mode-1 and mode-2 of $\mathcal{X}$  \\
  $\mathbf{X}_{(n)}\in   \mathbf{R}^{I_n \times \prod_{i \ne n} I_i} $                & the mode-$n$ unfolding of $\mathcal{X}$                    \\
  $\overline{\mathcal{X}}   \triangleq L(\mathcal{X})$                                & arbitrary invertible linear transform of $\mathcal{X}$ \\
  $L^{-1}(\mathcal{X})$                                                      & inverse linear transform of $\mathcal{X}$      \\
  $\operatorname{bdiag}(\mathcal{X}) $                                        & the block diagonal matrix of $\mathcal{X}$     \\
  $ \times_{n}$                                                              & the mode-$n$ product                           \\
  $\triangle$                                                                & the face-wise product                          \\
  $ *_{L}$                                                                   & the transforms $L$ based t-product             \\ \hline\hline
  \end{tabular}%
  }
\end{table}

\subsection{Notations}

In this paper, we use lowercase boldface (e.g. $\mathbf{x}$)  to denote vectors, uppercase boldface (e.g. $\mathbf{X}$) to denote matrices, and calligraphic letters (e.g. $\mathcal{X}$) to denote tensors. The ﬁelds of real and complex numbers are denoted as $\mathbf{R}$ and $\mathbf{C}$, respectively. We use $[d] = \{1,2,\cdots,d-1,d\}$ to denote the set of the first $d$ natural numbers, $\left\langle \cdot  \right\rangle$ to denote the expectation of a random variable, $\operatorname{tr}(\cdot)$ to denote the trace of a matrix, and $\mathbf{I}_{N}$ to denote an identity matrix sized $N \times N$. The diagonal matrix formed by a vector $\mathbf{x}$ is denoted as $\operatorname{diag}(\mathbf{x})$. For an order-$d$ tensor $\mathcal{X}\in \mathbf{R}^{I_1 \times  \cdots \times I_d}$, we use $\mathcal{X}_{i_1, \cdots, i_d}=\mathcal{X} (i_1, \cdots, i_d)$ to denote the $(i_1, \cdots, i_d)$-th entry of $\mathcal{X}$. We use $\mathcal{X}_{:,i_2,\cdots, i_d}=\mathcal{X}(:,i_2,\cdots,i_d)$ to denote the mode-1 fiber of $\mathcal{X}$ and $\mathcal{X}_{i_1,:,\cdots, i_d}=\mathcal{X}(i_1,:,\cdots,i_d)$ to denote the mode-2 fiber of $\mathcal{X}$. The slice along mode-1 and mode-2 of $\mathcal{X}$ is denoted as $\mathbf{X}^{(i_3, \cdots,i_d)} = \mathcal{X}(:,:,i_3, \cdots,i_d)$ and $\mathbf{X}^{j} \in \mathbf{R}^{I_{1} \times I_{2}}$ denotes the $j$-th slice along mode-1 and mode-2 of $\mathcal{X}$, i.e., $\mathbf{X}^{j}=\mathcal{X}\left(:,:, i_{3}, \cdots, i_{d}\right)$, where $j=\left(i_{d}-1\right) I_{3}\times \cdots \times I_{d-1}+\cdots+\left(i_{4}-1\right) I_{3}+i_{3}$. The mode-$n$ unfolding of $\mathcal{X}$ is denoted as $\mathbf{X}_{(n)}\in \mathbf{R}^{I_n \times \prod_{i \ne n} I_i} $ and the mode-$n$ product of matrix $\mathbf{U}$ and tensor $\mathcal{X}$ is denoted as $\mathcal{X} \times_{n} \mathbf{U}$, the meaning of which is  $\mathcal{Y}=\mathcal{X} \times_{n} \mathbf{U} \Leftrightarrow \mathbf{Y}_{(n)}=\mathbf{U} \cdot \mathbf{X}_{(n)} $.  The Frobenius norm of $\mathcal{X}$ is $\left\|\mathcal{X}\right\|_{F} =  \sqrt{\sum_{i_1 \cdots i_d}\left|\mathcal{X}_{i_1, \cdots, i_d}\right|^{2}}$. Table \ref{tab:pre1} summarizes the notations we use in this paper.

\subsection{Order-$d$ t-SVD Framework}

In this subsection, we briefly introduce the order-$d$ t-SVD framework. For more details, please refer to \cite{HTNN,orderDtSVD1,orderDtSVD2,orderDtSVD3}.

\begin{mydef}[Block diagonal matrix \cite{HTNN}]
  Given an order-$d$ tensor $\mathcal{X}\in \mathbf{R}^{I_1 \times I_2 \times  \cdots \times I_d}$, the block diagonal matrix of $\mathcal{X}$ is formed by 
  \begin{equation}
    \operatorname{bdiag}(\mathcal{X})=\operatorname{diag}\left(\mathbf{X}^{1}, \mathbf{X}^{2}, \cdots, \mathbf{X}^{J-1}, \mathbf{X}^{J}\right), 
  \end{equation}
 where $J=I_{3} \times \cdots \times I_{d}$ and $\operatorname{bdiag}(\mathcal{X}) \in \mathbf{R}^{I_{1} I_{3} \cdots I_{d} \times I_{2} I_{3} \cdots I_{d}}$.
\end{mydef}

\begin{mydef}[Face-wise product \cite{HTNN}]
  The face-wise product of two order-$d$ tensors $\mathcal{X}$ and $\mathcal{Y}$ is defined as 
  \begin{equation}
    \mathcal{Z}=\mathcal{X} \triangle \mathcal{Y} \Leftrightarrow \operatorname{bdiag}(\mathcal{Z})=\operatorname{bdiag}(\mathcal{X}) \cdot \operatorname{bdiag}(\mathcal{Y}). \label{face wide prod}
  \end{equation}
\end{mydef}

From (\ref{face wide prod}), we know that the face-wise product of two order-$d$ tensors $\mathcal{X}$ and $\mathcal{Y}$ can also be expressed as 
\begin{equation}
  \mathbf{Z}^{(i_3,\cdots, i_d)} = \mathbf{X}^{(i_3,\cdots, i_d)} \mathbf{Y}^{(i_3,\cdots, i_d)}, \quad \forall i_3,\cdots,i_d,
\end{equation}
where $\mathbf{Z}^{(i_3,\cdots, i_d)}$, $\mathbf{X}^{(i_3,\cdots, i_d)}$, and $ \mathbf{Y}^{(i_3,\cdots, i_d)}$ are the slice along mode-1 and mode-2 of $\mathcal{Z}$, $\mathcal{X}$, and $\mathcal{Y}$, respectively.

\begin{mydef}[Arbitrary invertible linear transform \cite{HTNN}]
  Given an order-$d$ tensor $\mathcal{X}\in \mathbf{R}^{I_1 \times I_2 \times  \cdots \times I_d}$, the arbitrary invertible linear transform of $\mathcal{X}$ is $\overline{\mathcal{X}} \triangleq L(\mathcal{X})$, where $L(\mathcal{X})=\mathcal{X} \times_{3} \mathbf{U}_{I_{3}} \times_{4} \mathbf{U}_{I_{4}} \cdots \times_{d} \mathbf{U}_{I_{d}}$ and $\{\mathbf{U}_{I_{j}} \in \mathbf{C}^{I_{j} \times I_{j}} \}_{j=3}^{d}$ are the transform matrices such as the Discrete Fourier Transformation (DFT) matrices. $L^{-1}(\mathcal{X})$ is the inverse linear transform of $\mathcal{X}$, i.e., $L^{-1}(\mathcal{X})=\mathcal{X} \times_{d} \mathbf{U}_{I_{d}}^{-1} \times_{d-1} \mathbf{U}_{I_{d}-1}^{-1} \cdots \times_{3} \mathbf{U}_{I_{3}}^{-1}$. 
\end{mydef}

From the definition of the arbitrary invertible linear transform, we know that $L^{-1}(L(\mathcal{X}))=\mathcal{X}$.

\begin{mydef}[Order-$d$ t-product \cite{HTNN}]
  The invertible linear transform $L$ based t-product of two given tensors $\mathcal{X} \in \mathbf{R}^{I_{1} \times l \times \cdots \times I_{d}}$ and $\mathcal{Y} \in \mathbf{R}^{l \times I_{2}\times  \cdots \times I_{d}}$ is defined as
  \begin{equation}
    \mathcal{Z} = \mathcal{X} *_L \mathcal{Y} = L^{-1} (\overline{\mathcal{X}} \triangle \overline{\mathcal{Y}}),
  \end{equation}

where $\mathcal{Z} \in \mathbf{R}^{I_1 \times I_2  \times \cdots \times I_d}$. 
\end{mydef}

From Definition II.1, II.2, II.3, II.4, we know that $\mathcal{Z} = \mathcal{X} *_L \mathcal{Y}$  is equivalent to $\operatorname{bdiag}(\overline{\mathcal{Z}}) =\operatorname{bdiag}(\overline{\mathcal{X}})\cdot\operatorname{bdiag}(\overline{\mathcal{Y}})$.

\begin{mydef}[Order-$d$ Conjugate transpose \cite{HTNN}]
  $\mathcal{X}^{\dagger} \in \mathbf{C}^{I_{2} \times I_{1}\times  \cdots \times I_{d}}$ is the conjugate transpose of the order-$d$ tensor $\mathcal{X} \in \mathbf{C}^{I_{1} \times I_{2}  \times\cdots \times I_{d}}$, if 
  \begin{equation}
    \overline{\mathcal{X}^{\dagger} }(:,:,i_3,\cdots, i_d) = \left(\overline{\mathcal{X}}\left(:,:,i_3,\cdots,i_d\right)\right)^{\dagger},
  \end{equation}
  for all $i_j \in [I_j], j \in \{3,...,d\}$. 
\end{mydef}

To obtain $\mathcal{X}^{\dagger}$, equivalently, we can conjugate transpose each frontal slice of $\mathcal{X}$, i.e. $\mathcal{X}(:,:,i_3,\cdots,i_d)$ and then reverse the order of the transposed frontal slices 2 through $I_k$ for each mode-$k$ where $k \in \{3,...,d \}$. 

In addition, from Definition II.1 and II.5, we know that $\operatorname{bdiag}(\overline{\mathcal{X}^{\dagger}}) = \operatorname{bdiag}(\overline{\mathcal{X}})^{\dagger}$.

\begin{mydef}[Identity tensor \cite{HTNN}]
  Order-$d$ identity tensor $\mathcal{I} \in \mathbf{R} ^{I \times I \times I_3 \times \cdots \times I_d}$ satisfies $\overline{\mathcal{I}}(:,:,i_3,\cdots,i_d) = \mathbf{I}_I$, for all $i_j \in [I_j], j \in \{3,...,d\}$.
\end{mydef}

\begin{mydef}[Orthogonal tensor \cite{HTNN}]
  Order-$d$ tensor $\mathcal{Q} \in \mathbf{C} ^{I \times I \times I_3 \times \cdots \times I_d}$ is said to be orthogonal, if it satisfies $\mathcal{Q}^{\dagger} *_{L} \mathcal{Q} = \mathcal{Q} *_{L} \mathcal{Q}^{\dagger} = \mathcal{I}$.
\end{mydef}

\begin{mydef}[Order-$d$ $f$-diagonal tensor\cite{HTNN}]
  Order-$d$ tensor $\mathcal{X} \in \mathbf{C} ^{I \times I \times I_3 \times \cdots \times I_d}$ is said to be $f$-diagonal, if $\mathcal{X}(:,:,i_3,\cdots,i_d)$ is a diagonal matrix, for all $i_j \in [I_j], j \in \{3,...,d\}$.
\end{mydef}

\begin{mytheo}[Order-$d$ t-SVD \cite{HTNN}]
  Given a real value tensor $\mathcal{X} \in \mathbf{R}^{I_1 \times I_2 \times I_3 \times \cdots \times I_d}$, it can be factorized as 
  \begin{equation}
    \mathcal{X} = \mathcal{U} *_{L} \mathcal{S} *_{L} \mathcal{V}^{\dagger},
  \end{equation}
  where $\mathcal{U} \in \mathbf{R}^{I_1 \times I_1 \times I_3 \times \cdots \times I_d}$ and $\mathcal{V} \in \mathbf{R}^{I_2 \times I_2 \times I_3 \times \cdots \times I_d}$ are orthogonal tensors, and $\mathcal{S} \in \mathbf{R}^{I_1 \times I_2 \times I_3 \times \cdots \times I_d}$ is a $f$-diagonal tensor which satisfies $\mathcal{S}(i_1,i_2,\cdots,i_d) = 0$ unless $i_1 = i_2$.
\end{mytheo}
We should note that the elements on the diagonal of the ﬁrst slice of $\mathcal{S}$, $\mathcal{S}(:,:,1,\cdots,1)$, have the decreasing property, i.e., 
\begin{equation}
  \mathcal{S}(1,1,1,\cdots,1) \ge \cdots \ge \mathcal{S}(m,m,1 \cdots ,1) \ge 0 ,
\end{equation}
where $m = \min (I_1,I_2)$. This property holds because
\begin{equation}
  \mathcal{S}(i, i, 1, \cdots, 1)=\frac{1}{\varphi} \sum_{i_{3}=1}^{I_{3}} \cdots \sum_{i_{d}=1}^{I_{d}} \overline{\mathcal{S}}\left(i, i, i_{3}, \cdots, i_{d}\right), \label{property1}
\end{equation}
where the constant $\varphi > 0$ is determined by the invertible linear transform $L$, and the elements on the diagonal of $\overline{\mathcal{S}}\left(i, i, i_{3}, \cdots, i_{d}\right)$ are referred to as the singular values of $\overline{\mathcal{X}}(:,:,i_3,\cdots,i_d)$.

\begin{mydef}[Order-$d$ tensor multi-rank \cite{HTNN}]
  Given a tensor $\mathcal{X} \in \mathbf{R}^{I_1 \times \cdots \times I_d}$, the multi-rank of $\mathcal{X}$ with respect to the invertible linear transform $L$ is a length-($I_3 \times \cdots \times I_d$) vector defined as 
  \begin{equation}
    \operatorname{rank}_{m}(\mathcal{X}) = \left( \operatorname{rank}\left(\overline{\mathbf{X}}^{1}\right), \cdots, \operatorname{rank}\left(\overline{\mathbf{X}}^{J}  \right) \right),
  \end{equation}
  where $J=I_{3} \times \cdots \times I_{d}$ , $\overline{\mathbf{X}}^{k}$ is the $k$-th slice along mode-1 and mode-2 of $\overline{\mathcal{X}} = L(\mathcal{X})$ and $\operatorname{rank}\left(\overline{\mathbf{X}}^{k}\right)$ is the rank of $\overline{\mathbf{X}}^{k}$.
\end{mydef}

\begin{mydef}[Order-$d$ tensor tubal rank \cite{HTNN}]
  Given a tensor $\mathcal{X} \in \mathbf{R}^{I_1 \times \cdots \times I_d}$ and let $\mathcal{S}$ be from the order-$d$ t-SVD of $\mathcal{X}$, the tubal rank of $\mathcal{X}$ is defined as 
  \begin{equation}
    \begin{aligned}
      \operatorname{rank}_{t}(\mathcal{X}) & =\sharp\{i: \mathcal{S}(i, i,:, \cdots,:) \neq \mathbf{0}\} \\
      & = \max_{k \in \left[J\right]} \,\operatorname{rank}\left(\overline{\mathbf{X}}^{k}\right),
      \end{aligned} \label{tubal rank}
  \end{equation}
  where $\sharp$ denotes the cardinality of a set, $J=I_{3} \times \cdots \times I_{d}$.
\end{mydef}
From the second equation of (\ref{tubal rank}), we know that the tensor tubal rank is the largest element of the tensor multi-rank.

\begin{mydef}[Order-$d$ skinny t-SVD \cite{HTNN}]
  Assume that the tubal rank of $\mathcal{X} \in \mathbf{R}^{I_1\times \cdots \times I_d}$ is $\operatorname{rank}_{t}(\mathcal{X})=r$, then the skinny t-SVD of $\mathcal{X}$ is $\mathcal{X} = \mathcal{U} *_{L} \mathcal{S} *_{L} \mathcal{V}^{\dagger}$, where $\mathcal{U} \in \mathbf{R}^{I_1\times r \times I_3 \times \cdots \times I_d}$ and $\mathcal{V} \in \mathbf{R}^{I_2\times r \times I_3 \times \cdots \times I_d}$ are orthogonal tensors, and $\mathcal{S} \in \mathbf{R}^{r\times r \times I_3 \times \cdots \times I_d}$ is a $f$-diagonal tensor.
\end{mydef}

\begin{mylemma}
  Given a tensor $\mathcal{X} \in \mathbf{R}^{I_1 \times I_2 \times \cdots \times I_d}$ whose tubal rank is $\operatorname{rank}_{t}(\mathcal{X})=r < \min(I_1,I_2)$ it can be factorized as the t-product of two smaller tensors, i.e., $\mathcal{X} = \mathcal{U} *_{L} \mathcal{V}^{\dagger}$, where $\mathcal{U} \in \mathbf{R}^{I_1 \times r \times \cdots \times I_d}$ and $\mathcal{V} \in \mathbf{R}^{I_2 \times r \times \cdots \times I_d}$.
\end{mylemma}

\begin{proof}
  According to Definition II.11, we know that an order-$d$ tensor $\mathcal{X} \in \mathbf{R}^{I_1 \times I_2 \times \cdots \times I_d}$ with tubal rank $\operatorname{rank}_{t}(\mathcal{X})=r < \min(I_1,I_2)$ can be factorized as $\mathcal{X} = \mathcal{U}_0 *_{L} \mathcal{S}_0 *_{L} \mathcal{V}_0^{\dagger}$, where $\mathcal{U}_0 \in \mathbf{R}^{I_1\times r \times I_3 \times \cdots \times I_d}$ and $\mathcal{V}_0 \in \mathbf{R}^{I_2\times r \times I_3 \times \cdots \times I_d}$ are orthogonal tensors, and $\mathcal{S}_0 \in \mathbf{R}^{r\times r \times I_3 \times \cdots \times I_d}$ is a $f$-diagonal tensor. In the transform domain, equivilently, we have $\operatorname{bdiag}(\overline{\mathcal{X}}) =\operatorname{bdiag}(\overline{\mathcal{U}_{0}} ) \cdot \operatorname{bdiag}( \overline{\mathcal{S}_{0}}) \cdot \operatorname{bdiag}(\overline{\mathcal{V}_{0}})^{\dagger}$. Let $\operatorname{bdiag}(\overline{\mathcal{U}}) = \operatorname{bdiag}(\overline{\mathcal{U}_{0}})  \cdot \operatorname{bdiag}(\overline{\mathcal{S}_{0}})^{\frac{1}{2}}$, and $\operatorname{bdiag}(\overline{\mathcal{V}}) = \operatorname{bdiag}(\overline{\mathcal{V}_{0}})  \cdot \operatorname{bdiag}(\overline{\mathcal{S}_{0}})^{\frac{1}{2}}$, we have 
  \begin{equation}
    \begin{aligned}
      & \quad \operatorname{bdiag}(\overline{\mathcal{U}}) \cdot \operatorname{bdiag}(\overline{\mathcal{V}})^{\dagger} \\
      &=  \operatorname{bdiag}(\overline{\mathcal{U}_{0}})  \cdot \operatorname{bdiag}(\overline{\mathcal{S}_{0}})^{\frac{1}{2}} \cdot \operatorname{bdiag}(\overline{\mathcal{S}_{0}})^{\frac{1}{2}} \cdot \operatorname{bdiag}(\overline{\mathcal{V}_{0}})^{\dagger}\\
      &= \operatorname{bdiag}(\overline{\mathcal{U}_{0}})  \cdot \operatorname{bdiag}(\overline{\mathcal{S}_{0}}) \cdot \operatorname{bdiag}(\overline{\mathcal{V}_{0}})^{\dagger}\\
      &=\operatorname{bdiag}(\overline{\mathcal{X}}),
      \end{aligned}
  \end{equation}
  which means $\operatorname{bdiag}(\overline{\mathcal{X}}) = \operatorname{bdiag}(\overline{\mathcal{U}}) \cdot \operatorname{bdiag}(\overline{\mathcal{V}})^{\dagger}$. Therefore, we have $\mathcal{X} = \mathcal{U} *_{L} \mathcal{V}^{\dagger}$ in the original domain, where $\mathcal{U} \in \mathbf{R}^{I_1 \times r \times \cdots \times I_d}$ and $\mathcal{V} \in \mathbf{R}^{I_2 \times r \times \cdots \times I_d}$. 
\end{proof}

\section{Low-Multi-Rank High-Order BRTF Model}

\begin{figure}[tbp]  
  \centering  
  \includegraphics[width=8cm]{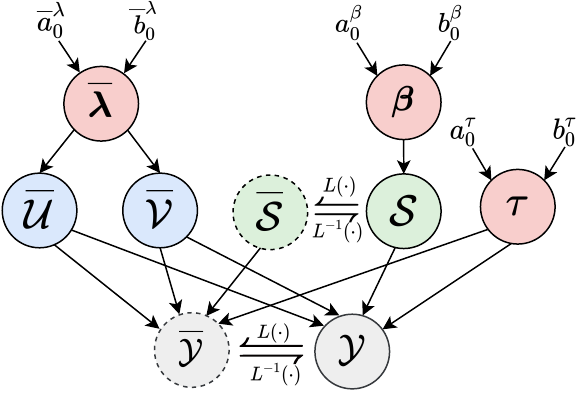} 
  \caption{Graphical model of the proposed LMH-BRTF model.}  
  \label{fig1}
\end{figure}

In this section, we first specify the LMH-BRTF model. Then, we derive a variational inference algorithm to efficiently estimate the model parameters.

\subsection{Model Speciﬁcation}
We assume that the observed corrupted tensor $\mathcal{Y} \in \mathbf{R}^{I_{1}\times I_2 \times \cdots \times I_{d}}$ can be decomposed into three parts \cite{BRTF}: the low-rank component $\mathcal{X} \in \mathbf{R}^{I_{1}\times I_2 \times \cdots \times I_{d}}$, the sparse component $\mathcal{S} \in \mathbf{R}^{I_{1}\times I_2 \times \cdots \times I_{d}}$, and the noise component $\mathcal{E} \in \mathbf{R}^{I_{1}\times I_2 \times \cdots \times I_{d}}$, that is, 
\begin{equation}
  \mathcal{Y} = \mathcal{X} + \mathcal{S} + \mathcal{E}, \label{sum assump}
\end{equation}
and our goal is to infer $\mathcal{X}$ from $\mathcal{Y}$. In the following, we model $\mathcal{X}$, $\mathcal{S}$, and $\mathcal{E}$ respectively and meanwhile introduce suitable priors accordingly for them. Then we specify the conditional distribution of $\mathcal{Y}$ and give the joint distribution of the LMH-BRTF model.

\textit{1) Probabilistic Modeling for $\mathcal{X}$}. By Lemma 1, we can factorize any order-$d$ tensor as the t-product of two smaller factor tensors. We decompose $\mathcal{X}$ as
\begin{equation}
  \mathcal{X} = \mathcal{U} *_{L} \mathcal{V}^{\dagger}=L^{-1} (\overline{\mathcal{U}} \triangle \overline{\mathcal{V}}^{\dagger}), \label{tsvd decom}
\end{equation}
where $\mathcal{U} \in \mathbf{R}^{I_1\times r \times I_3 \times \cdots \times I_d}$, $\mathcal{V} \in \mathbf{R}^{I_2\times r \times I_3 \times \cdots \times I_d}$, and $r \le \min(I_1,I_2)$ is the tubal rank of $\mathcal{X}$. In this way, we introduce a low-rank structure to $\mathcal{X}$.

In the transform domain, (\ref{tsvd decom}) is equivalent to 
\begin{equation}
  \overline{\mathcal{X}} = \overline{\mathcal{U}} \triangle \overline{\mathcal{V}}^{\dagger},
\end{equation}
that is,
\begin{equation}
  \overline{\mathbf{X}}^{k} = \overline{\mathbf{U}}^{k} \overline{\mathbf{V}}^{k \dagger},
\end{equation}
where $k \in [J]$, $J = I_3 \times \cdots \times I_d$. Since most of the definitions and operations of the order-$d$ t-SVD are actually in the transform domain, we directly specify the factor tensors, $\overline{\mathcal{U}}$ and $\overline{\mathcal{V}}$, in the transform domain.

Recall that one of our goals is to enable the model to automatically determine the multi-rank of the original domain tensor $\mathcal{X}$. However, multi-rank determination is a much more general and difﬁcult problem than tubal rank determination, as the tubal rank is just the largest element
of the multi-rank (see Definition II.10). In practice, the unknown multi-rank is generally considered as a hyper-parameter and may be tuned empirically, which should be too time-consuming as multi-rank is a length-($I_3 \times \cdots \times I_d$) vector for a tensor $\mathcal{X} \in \mathbf{R}^{I_1 \times \cdots \times I_d}$. Another way is to use heuristic multi-rank determination strategies such as that in \cite{heuristic MR deter}. However, the strategy in \cite{heuristic MR deter} fails when noises are significant. In contrast, we attempt to determine the multi-rank with the help of sparse Bayesian learning \cite{ARD}. Specifically, according to Definition II.9, determining the multi-rank of $\mathcal{X}$ is equivalent to determining the rank of each $\overline{\mathbf{X}}^{k} = \overline{\mathbf{U}}^{k} \overline{\mathbf{V}}^{k \dagger}$, which is controlled by the number of columns in $\overline{\mathbf{U}}^{k}$ and $\overline{\mathbf{V}}^{k}$, $ \forall k \in [J], J=I_3 \times \cdots \times I_d$. We introduce sparsity for each column in $\overline{\mathbf{U}}^{k}$ and $\overline{\mathbf{V}}^{k}$, so that the rank of each $\overline{\mathbf{X}}^{k}$ can be determined by removing those unnecessary columns in  $\overline{\mathbf{U}}^{k}$ and $\overline{\mathbf{V}}^{k}$. In particular, we place the ARD prior \cite{ARD} over the factor tensors $\overline{\mathcal{U}}$ and $\overline{\mathcal{V}}$ as follows:
\begin{align}
  p(\overline{\mathcal{U}}\mid \overline{\boldsymbol{\lambda}}) &= \prod_{i_1 r i_3 \cdots i_d} \mathcal{N}\left(\overline{\mathcal{U}}_{i_1,r,\cdots,i_d}\mid 0 , \overline{\lambda}_{r}^{(i_3,\cdots,i_d)-1}\right), \\
  p(\overline{\mathcal{V}}\mid \overline{\boldsymbol{\lambda}}) & =\prod_{i_2 r i_3 \cdots i_d }\mathcal{N}\left(\overline{\mathcal{V}}_{i_2,r,\cdots,i_d}\mid 0 , \overline{\lambda}_{r}^{(i_3,\cdots,i_d)-1}\right),
\end{align}
where $\overline{\lambda}_{r}^{(i_3,\cdots,i_d)}$ is the precision of the $r$-th column in $\overline{\mathbf{U}}^{(i_3,\cdots,i_d)}$ and $\overline{\mathbf{V}}^{(i_3,\cdots,i_d)}$. Further, to simply our analysis, we apply the conjugate distribution of the Gaussian distribution, i.e., the Gamma distribution, to each $\overline{\lambda}_{r}^{(i_3,\cdots,i_d)}$, which is given by
\begin{equation}
  p(\overline{\boldsymbol{\lambda}}) = \prod_{ri_3\cdots i_d}\operatorname{Ga}(\overline{\lambda}_{r}^{(i_3,\cdots,i_d)}\mid \overline{a}_{0}^{\lambda},\overline{\, b}_{0}^{\lambda}), \label{lambda prior}
\end{equation}
where $\operatorname{Ga}(x \mid a,b) = \frac{b^{a}x^{a-1}e^{-bx}}{\Gamma(a)}$ with $\Gamma(a)$ being the Gamma function, and $a$ and $b$ are the hyper-parameters.

With this form of prior specification, those $\overline{\lambda}_{r}^{(i_3,\cdots,i_d)}$ corresponding to the unnecessary columns in $\overline{\mathbf{U}}^{(i_3,\cdots,i_d)}$ and $\overline{\mathbf{V}}^{(i_3,\cdots,i_d)}$ will tend to have large values, so as to drive these unnecessary columns to be $\mathbf{0}$. In this way, we are able to determine the multi-rank of $\mathcal{X}$  by retaining only the minimum number of columns in each $\overline{\mathbf{U}}^{(i_3,\cdots,i_d)}$ and $\overline{\mathbf{V}}^{(i_3,\cdots,i_d)}$.

\textit{2) Probabilistic Modeling for $\mathcal{S}$.} The sparse component $\mathcal{S}$ captures the sparse noise in the observed corrupted tensor $\mathcal{Y}$ in the original domain, and such noise is also known as the outlier. Mathematically, most elements in $\mathcal{S}$ are 0, with only a few having large magnitudes. As the sparse component $\mathcal{S}$ has its specific physical meaning in the original domain and the physical meaning is not easy to be described in the transform domain, we specify the sparse component $\mathcal{S}$ in the original domain. Specifically, we introduce independent Gaussian distribution for each element of $\mathcal{S}$, which is given by
\begin{equation}
  p(\mathcal{S}\mid \boldsymbol{\beta}) = \prod_{i_1\cdots i_d} \mathcal{N}\left(\mathcal{S}_{i_1, \cdots, i_d} \mid 0, \beta_{i_1,\cdots ,i_d}^{-1}\right),
\end{equation}
where $\boldsymbol{\beta} = \{\beta_{i_1,\cdots,i_d}\}$, and $\beta_{i_1,\cdots,i_d}$ is the precision of the $(i_1,\cdots, i_d)$-th element of $\mathcal{S}$. Further, we impose on each $\beta_{i_1,\cdots,i_d}$ the conjugate distribution of the Gaussian distribution, i.e., the Gamma distribution, which is given by
\begin{equation}
  p(\boldsymbol{\beta}) = \prod_{i_1\cdots i_d} \mathrm{Ga}\left(\beta_{i_1, \cdots, i_d} \mid a_{0}^{\beta}, b_{0}^{\beta}\right). \label{beta prior}
\end{equation}

Note that as $\beta_{i_1,\cdots,i_d}$ becomes large, the corresponding $\mathcal{S}_{i_1, \cdots, i_d}$ will tend to be 0. By encouraging most of the precision variables to take on large values, the corresponding $\mathcal{S}$ is capable of characterizing the sparse noise.

\textit{3) Probabilistic Modeling for $\mathcal{E}$}. Similar to the sparse component $\mathcal{S}$, the noise component $\mathcal{E}$ also has its physical meaning in the origianl domain that it absorbs the noises except the sparse noise. We also specify the noise component $\mathcal{E}$ in the original domain. In this paper, we use the independent zero-mean Gaussian distribution to model the noise component $\mathcal{E}$. Different from the modeling of the sparse component $\mathcal{S}$, we assume that the precision of each element in $\mathcal{E}$ is identical, i.e.,
\begin{equation}
  p(\mathcal{E}\mid \tau ) = \prod_{i_1\cdots i_d} \mathcal{N}(\mathcal{E}_{i_1,\cdots, i_d} \mid 0,\tau^{-1}),
\end{equation}
where $\tau$ is the noise precision. Similar to (\ref{lambda prior}) and (\ref{beta prior}), we impose on $\tau$ the conjugate distribution of the Gaussian distribution, i.e., the Gamma distribution, which is given by
\begin{equation}
  p(\tau) = \operatorname{Ga}\left(\tau \mid a_{0}^{\tau}, b_{0}^{\tau}\right).
\end{equation}

In general, both $a_{0}^{\tau}$ and $ b_{0}^{\tau}$ are set to very small values to introduce non-informative priors.

\textit{4) Conditional Distribution of $\mathcal{Y}$}. By (\ref{sum assump}), (\ref{tsvd decom}) and the modeling above, given the model parameters, the conditional distribution of the observed corrupted tensor $\mathcal{Y}$ is
\begin{align}
  &p(\mathcal{Y}\mid \overline{\mathcal{U}},\overline{\mathcal{V}},\mathcal{S},\tau) = \notag \\
  & \prod_{i_1  \cdots i_d} \mathcal{N}\left(\mathcal{Y}_{i_1, \cdots, i_d}\mid L^{-1}\left(\overline{\mathcal{U}} \triangle \overline{\mathcal{V}}^{\dagger} \right)_{i_1 ,\cdots, i_d}+\mathcal{S}_{i_1, \cdots, i_d}, \tau^{-1} \right).
\end{align}

Within the algebraic framework of the order-$d$ t-SVD, the original domain and the transform domain are connected by the invertible linear transform $L$. Through the linear transform $L$, we know that
\begin{align}
  \overline{\mathcal{Y}} &=\overline{\mathcal{X}} + \overline{\mathcal{S}} + \overline{\mathcal{E}},\\
  &= \overline{\mathcal{U}} \triangle \overline{\mathcal{V}}^{\dagger}  + \overline{\mathcal{S}} + \overline{\mathcal{E}}.
\end{align}
Therefore, the conditional distribution can also be given in the transform domain as 
\begin{align}
  &p(\overline{\mathcal{Y}}\mid \overline{\mathcal{U}},\overline{\mathcal{V}},\overline{\mathcal{S}},\tau) = \notag \\
  &\prod_{i_1\cdots i_d} \mathcal{N}\left(\overline{\mathcal{Y}}_{i_1,\cdots, i_d}\mid \overline{\mathcal{U}}_{i_1,:,\cdots,i_d} \cdot \overline{\mathcal{V}}^{\dagger}_{i_2,:,\cdots,i_d}+\overline{\mathcal{S}}_{i_1, \cdots, i_d}, \varphi \tau^{-1} \right),
\end{align}
where the constant $\varphi$ is determined by the invertible linear transform $L$. For example, $\varphi = I_3\times \cdots \times I_d$ if $L$ is the discrete Fourier transformation.

\textit{5) Joint Distribution of LMH-BRTF}. Based on the model specification above, we can obtain the joint distribution in the original domain and in the transform domain as  
\begin{align}
  p(\mathcal{Y}, \mathbf{\Theta}) &= p(\mathcal{Y}\mid \overline{\mathcal{U}},\overline{\mathcal{V}},\mathcal{S},\tau)p(\overline{\mathcal{U}}\mid \overline{\boldsymbol{\lambda}})p(\overline{\mathcal{V}} \mid \overline{\boldsymbol{\lambda}})\notag \\
  & \qquad \qquad  \qquad \quad \times  p(\mathcal{S} \mid \boldsymbol{\beta})p(\overline{\boldsymbol{\lambda}})p(\boldsymbol{\beta})p(\tau),
\end{align}
and 
\begin{align}
  p(\overline{\mathcal{Y}}, \mathbf{\Theta}) &= p(\overline{\mathcal{Y}}\mid \overline{\mathcal{U}},\overline{\mathcal{V}},\overline{\mathcal{S}},\tau)p(\overline{\mathcal{U}}\mid \overline{\boldsymbol{\lambda}})p(\overline{\mathcal{V}} \mid \overline{\boldsymbol{\lambda}})\notag \\
  & \qquad \qquad  \qquad \quad \times p(\mathcal{S} \mid \boldsymbol{\beta})p(\overline{\boldsymbol{\lambda}})p(\boldsymbol{\beta})p(\tau),
\end{align}
respectively, where $\mathbf{\Theta} = \{\overline{\mathcal{U}},\overline{\mathcal{V}},\overline{\boldsymbol{\lambda}},\mathcal{S},\boldsymbol{\beta},\tau\}$ is the set of latent variables in the LMH-BRTF model. Correspondingly, the log of the joint distribution in the original domain is given by
\begin{align}
  &\ln p(\mathcal{Y}, \mathbf{\Theta}) = \frac{I_1 \cdots I_d}{2} \ln \tau \notag\\
  &+ \sum_{i_1 \cdots i_d} \left[  -\frac{\tau}{2} \left( \mathcal{Y}_{i_1, \cdots, i_d} - \left( \overline{\mathcal{U}} \triangle \overline{\mathcal{V}}^{\dagger} \right)_{i_1 ,\cdots, i_d} - \mathcal{S}_{i_1, \cdots, i_d}  \right)^2 \right]\notag\\
  &+ \sum_{i_1 r i_3\cdots i_d} \left[\frac{1}{2} \ln \overline{\lambda}_{r}^{(i_3,\cdots,i_d)} - \frac{\overline{\lambda}_{r}^{(i_3,\cdots,i_d)}}{2}\overline{\mathcal{U}}_{i_1,r,\cdots,i_d}^{2} \right] \notag\\
  &+ \sum_{i_2 r i_3\cdots i_d} \left[\frac{1}{2} \ln \overline{\lambda}_{r}^{(i_3,\cdots,i_d)} - \frac{\overline{\lambda}_{r}^{(i_3,\cdots,i_d)}}{2}\overline{\mathcal{V}}_{i_2,r,\cdots,i_d}^{2} \right]\notag\\
  &+ \sum_{i_1 \cdots i_d} \left[\frac{1}{2} \ln \beta_{i_1, \cdots, i_d } - \frac{\beta_{i_1, \cdots,i_d}}{2}\mathcal{S}_{i_1,\cdots,i_d}^{2} \right] \notag \\
  &+\sum_{ri_3\cdots i_d} \left[(\overline{a}_{0}^{\lambda}-1)\ln \overline{\lambda}_{r}^{(i_3,\cdots,i_d)} -  \overline{\, b}_{0}^{\lambda} \overline{\lambda}_{r}^{(i_3,\cdots,i_d)}   \right]  \notag \\
  &+ \sum_{i_1  \cdots i_d}  \left[(a_{0}^{\beta}-1) \ln \beta_{i_1, \cdots, i_d} - b_{0}^{\beta} \beta_{i_1, \cdots, i_d} \right]\notag\\
  &+(a_{0}^{\tau}-1) \ln \tau - b_{0}^{\tau} \tau + \text{const}, \label{original log}
\end{align}
and that in the transform domain is 
\begin{align}
  &\ln p(\overline{\mathcal{Y}}, \mathbf{\Theta}) = \frac{I_1 \cdots I_d}{2} \ln \tau \notag \\
  &+ \sum_{i_1 \cdots i_d} \left[ -\frac{\tau}{2 \varphi} \left(\overline{\mathcal{Y}}_{i_1,\cdots ,i_d} - \overline{\mathcal{U}}_{i_1,:,\cdots,i_d} \cdot \overline{\mathcal{V}}^{\dagger}_{i_2,:,\cdots,i_d}-\overline{\mathcal{S}}_{i_1, \cdots, i_d}  \right)^2 \right] \notag\\
  &+ \sum_{i_1 r i_3\cdots i_d} \left[\frac{1}{2} \ln \overline{\lambda}_{r}^{(i_3,\cdots,i_d)} - \frac{\overline{\lambda}_{r}^{(i_3,\cdots,i_d)}}{2}\overline{\mathcal{U}}_{i_1,r,\cdots,i_d}^{2} \right]\notag\\
  &+ \sum_{i_2 r i_3\cdots i_d} \left[\frac{1}{2} \ln \overline{\lambda}_{r}^{(i_3,\cdots,i_d)} - \frac{\overline{\lambda}_{r}^{(i_3,\cdots,i_d)}}{2}\overline{\mathcal{V}}_{i_2,r,\cdots,i_d}^{2} \right]\notag\\
  &+ \sum_{i_1 \cdots i_d} \left[\frac{1}{2} \ln \beta_{i_1, \cdots, i_d } - \frac{\beta_{i_1, \cdots,i_d}}{2}\mathcal{S}_{i_1,\cdots,i_d}^{2} \right] \notag \\
  &+\sum_{ri_3\cdots i_d} \left[(\overline{a}_{0}^{\lambda}-1)\ln \overline{\lambda}_{r}^{(i_3,\cdots,i_d)} -  \overline{\, b}_{0}^{\lambda} \overline{\lambda}_{r}^{(i_3,\cdots,i_d)} \right]    \notag \\
  &+ \sum_{i_1  \cdots i_d} \left[(a_{0}^{\beta}-1) \ln \beta_{i_1, \cdots, i_d} - b_{0}^{\beta} \beta_{i_1, \cdots, i_d} \right]\notag\\
  &+(a_{0}^{\tau}-1) \ln \tau - b_{0}^{\tau} \tau + \text{const}. \label{transform log}
\end{align}

The graphical model of the LMH-BRTF model is shown in Fig.\ref{fig1}.

\subsection{Model Learning via Variational Inference}

With the model specification above, theoretically LMH-BRTF can be learned by calculating the posterior of the variables $p(\mathbf{\Theta} \mid \mathcal{Y}) = \frac{p(\mathcal{Y},\mathbf{\Theta})}{\int p(\mathcal{Y},\mathbf{\Theta})d\mathbf{\Theta}}$ (or $p(\mathbf{\Theta} \mid \overline{\mathcal{Y}})$). However, as the posterior distribution is analytically intractable, we resort to variational inference methods \cite{VB1,VB2} to approximate the posterior distribution. Specifically, we construct a variational distribution $q(\mathbf{\Theta})$ and measure its dissimilarity to the true posterior $p(\mathbf{\Theta}\mid \mathcal{Y})$ by the KL divergence $\operatorname{KL}(q(\mathbf{\Theta})|| p(\mathbf{\Theta}\mid \mathcal{Y}))$. Then our goal is to 
\begin{equation}
  \min \text{KL}[q(\mathbf{\Theta}) \| p(\mathbf{\Theta} \mid \mathcal{Y})], \label{KL div}
\end{equation}
which is equivalent to maximizing the evidence lower bound $\mathcal{L}(q)=\int q(\mathbf{\Theta}) \log p(\mathcal{Y}, \mathbf{\Theta}) d \mathbf{\Theta}-\int q(\mathbf{\Theta}) \log q(\mathbf{\Theta}) d \mathbf{\Theta}$.

For the variational distribution $q(\mathbf{\Theta})$, the mean-field assumption is generally introduced, which assumes that the elements of $\mathbf{\Theta}$ can be partitioned into disjoint groups $\{\mathbf{\Theta}_j\}$, and $q(\mathbf{\Theta})$ can be factorized with respect to these disjoint groups:  
\begin{equation}
  q(\mathbf{\Theta}) = \prod_j q(\mathbf{\Theta}_j). \label{mean-field assumption}
\end{equation}

By substituting (\ref{mean-field assumption}) into (\ref{KL div}) and minimize the KL divergence between the variational posterior $q(\mathbf{\Theta})$ and the true posterior $ p(\mathbf{\Theta} \mid \mathcal{Y})$, we can obtain the general expression for the optimal solution of (\ref{KL div}) with regard to the $j$-th variable set $\boldsymbol{\Theta}_{j}$, which is given by
\begin{equation}
\ln q_{j}\left(\boldsymbol{\Theta}_{j}\right) \propto\langle\ln p(\mathcal{Y}, \boldsymbol{\Theta})\rangle_{\boldsymbol{\Theta} \backslash \boldsymbol{\Theta}_{j}}, \label{VB optimal}
\end{equation}
where $\langle \rangle_{\boldsymbol{\Theta} \backslash \boldsymbol{\Theta}_{j}}$ denotes an expectation with respect to the variational distribution $q(\mathbf{\Theta})$ except the latent variables $\boldsymbol{\Theta}_{j}$. For the LMH-BRTF model, each $\boldsymbol{\Theta}_{j}$ can be specified as one of the elements in $\{\overline{\mathcal{U}},\overline{\mathcal{V}},\overline{\boldsymbol{\lambda}},\mathcal{S},\boldsymbol{\beta},\tau\}$, which is the set of latent variables in LMH-BRTF and the relationship among these latent variables is shown in Fig.\ref{fig1}. We can then re-write (\ref{mean-field assumption}) as 
\begin{equation}
  q(\mathbf{\Theta}) = q(\overline{\mathcal{U}})q(\overline{\mathcal{V}})q(\overline{\boldsymbol{\lambda}})q\left(\mathcal{S}\right)q(\boldsymbol{\beta}) q(\tau). \label{LMH-BRTF mean-field}
\end{equation}

Note that the LMH-BRTF model has both latent variables in the original domain and latent variables in the transform domain, simply using the log of the joint distribution in the original domain (\ref{original log}) or that in the transform domain (\ref{transform log}) is not sufficient when applying (\ref{VB optimal}) to develop the variational inference algorithm for LMH-BRTF. In order to effectively estimate the parameters of LMH-BRTF, we propose a variational inference algorithm that crosses the original domain and the transform domain. In the following, we present the variational inference algorithm for LMH-BRTF in detail.

\textit{1) Variational Posterior of $\overline{\mathcal{U}}$ and $\overline{\mathcal{V}}$}. As the factor tensor $\overline{\mathcal{U}}$ is specified in the transform domain, we use the log of the joint distribution in the transfrom domain (\ref{transform log}). Let $\boldsymbol{\Theta}_{j} = \overline{\mathcal{U}}$ and apply (\ref{VB optimal}), the log of the optimized factor is given by
\begin{align}
  & \ln q\left(\overline{\mathcal{U}} \right) = \sum_{i_1 i_3 \cdots i_d} \left[-\frac{1}{2}\overline{\mathcal{U}}_{i_1,:,\cdots,i_d} \left(\frac{\left\langle\tau\right\rangle}{\varphi} \left\langle\overline{\mathbf{V}}^{\dagger (i_3,\cdots,i_d)}\overline{\mathbf{V}}^{(i_3,\cdots,i_d)}  \right\rangle \right. \right. \notag \\ 
  &  \left. \qquad   +\operatorname{diag}(\left\langle\overline{\boldsymbol{\lambda}}^{(i_3,\cdots,i_d)}\right\rangle)\right)\overline{\mathcal{U}}^{\top}_{i_1,:,\cdots,i_d} \notag \\
  & \left. + \frac{\left\langle\tau \right\rangle}{\varphi}\left( \overline{\mathcal{Y}}_{i_1,:,\cdots,i_d} - \overline{\mathcal{S}}_{i_1,:,\cdots,i_d} \right)  \left\langle \overline{\mathbf{V}}^{(i_3,\cdots,i_d)}\right\rangle \overline{\mathcal{U}}^{\top}_{i_1,:,\cdots,i_d} \right]\notag \\
  & + \text{const}.  
\end{align}

Accordingly, the variational posterior of $q\left(\overline{\mathcal{U}}\right)$ follows the product of Gaussian distributions and can be obtained as 
\begin{equation}
  q(\overline{\mathcal{U}} ) = \prod_{i_1 i_3 \cdots i_d}\mathcal{N}\left(\overline{\mathcal{U}}_{i_1,:,\cdots,i_d} \mid \left\langle\overline{\mathcal{U}}_{i_1,:,\cdots,i_d} \right\rangle, \boldsymbol{\Sigma}_{\overline{u}}^{(i_3,\cdots,i_d)}\right),
\end{equation}
where
\begin{align}
  \left\langle\overline{\mathcal{U}}_{i_1,:,\cdots,i_d} \right\rangle & = \frac{\left\langle\tau \right\rangle}{\varphi}\left( \overline{\mathcal{Y}}_{i_1,:,\cdots,i_d} - \overline{\mathcal{S}}_{i_1,:,\cdots,i_d} \right)\notag \\
  & \qquad \qquad \qquad \quad \left\langle \overline{\mathbf{V}}^{(i_3,\cdots,i_d)}\right\rangle \boldsymbol{\Sigma}_{\overline{u}}^{(i_3,\cdots,i_d)}, \label{update U}\\
  \boldsymbol{\Sigma}_{\overline{u}}^{(i_3,\cdots,i_d)}&= \left(\frac{\left\langle\tau\right\rangle}{\varphi}\left\langle\overline{\mathbf{V}}^{\dagger (i_3,\cdots,i_d)}\overline{\mathbf{V}}^{(i_3,\cdots,i_d)}  \right\rangle \right. \notag \\
  &\left. \qquad \qquad \qquad+ \operatorname{diag}(\left\langle\overline{\boldsymbol{\lambda}}^{(i_3,\cdots,i_d)}\right\rangle)\right)^{-1}. \label{update sigma u}
\end{align}

Similarly, the variational posterior of $q\left(\overline{\mathcal{V}} \right)$ can be obtained as 
\begin{equation}
  q(\overline{\mathcal{V}}) = \prod_{i_2\cdots i_d}\mathcal{N}\left(\overline{\mathcal{V}}_{i_2,:,\cdots,i_d} \mid \left\langle\overline{\mathcal{V}}_{i_2,:,\cdots,i_d} \right\rangle,\boldsymbol{\Sigma}_{\overline{v}}^{{(i_3,\cdots,i_d})}\right),
\end{equation}
where
\begin{align}
  \left\langle\overline{\mathcal{V}}_{i_2,:,\cdots,i_d} \right\rangle & = \frac{\left\langle\tau \right\rangle}{\varphi}\left( \overline{\mathcal{Y}}_{:,i_2,\cdots, i_d} - \overline{\mathcal{S}}_{:,i_2,\cdots,i_d} \right)^{\dagger} \notag \\
  & \qquad \qquad\qquad \quad  \left\langle \overline{\mathbf{U}}^{(i_3,\cdots,i_d)}\right\rangle \boldsymbol{\Sigma}_{\overline{v}}^{{(i_3,\cdots,i_d})}, \label{update V}\\  
 \boldsymbol{\Sigma}_{\overline{v}}^{{(i_3,\cdots,i_d})}&= \left(\frac{\left\langle\tau\right\rangle}{\varphi}\left\langle\overline{\mathbf{U}}^{\dagger (i_3,\cdots,i_d)}\overline{\mathbf{U}}^{(i_3,\cdots,i_d)}  \right\rangle \right. \notag \\
 &\left. \qquad \qquad \qquad+ \operatorname{diag}(\left\langle\overline{\boldsymbol{\lambda}}^{(i_3,\cdots,i_d)}\right\rangle)\right)^{-1}. \label{update sigma v}
\end{align}

The two expectation terms $\left\langle\overline{\mathbf{U}}^{\dagger (i_3,\cdots,i_d)}\overline{\mathbf{U}}^{(i_3,\cdots,i_d)}  \right\rangle $ and $\left\langle\overline{\mathbf{V}}^{\dagger (i_3,\cdots,i_d)}\overline{\mathbf{V}}^{(i_3,\cdots,i_d)}  \right\rangle $ can be computed by

\begin{align}
  & \left\langle\overline{\mathbf{U}}^{\dagger (i_3,\cdots,i_d)}\overline{\mathbf{U}}^{(i_3,\cdots,i_d)}  \right\rangle \notag \\  
  & \qquad  =I_1 \boldsymbol{\Sigma}_{\overline{u}}^{(i_3,\cdots,i_d)} + \left\langle\overline{\mathbf{U}}^{(i_3,\cdots,i_d)}  \right\rangle^{\dagger} \left\langle \overline{\mathbf{U}}^{(i_3,\cdots,i_d)}  \right\rangle,  \label{expUU}\\  
  & \left\langle\overline{\mathbf{V}}^{\dagger (i_3,\cdots,i_d)}\overline{\mathbf{V}}^{(i_3,\cdots,i_d)}  \right\rangle \notag \\  
  & \qquad =I_2 \boldsymbol{\Sigma}_{\overline{v}}^{{(i_3,\cdots,i_d})} + \left\langle\overline{\mathbf{V}}^{(i_3,\cdots,i_d)}  \right\rangle^{\dagger} \left\langle \overline{\mathbf{V}}^{(i_3,\cdots,i_d)}  \right\rangle. \label{expVV}
\end{align}

\textit{2) Variational Posterior of $\overline{\boldsymbol{\lambda}}$.} Again, using (\ref{VB optimal}) we have
\begin{align}
  &\ln q(\overline{\boldsymbol{\lambda}})   = \sum_{r i_3 \cdots i_d} \left[\left(\overline{a}_0^{\lambda}+\frac{I_1 + I_2}{2}-1 \right) \ln \bar{\lambda}_r^{(i_3,\cdots,i_d)} - \left(   \overline{\, b}_0^{\lambda} + \right.\right. \notag \\  
  &\left. \frac{\left(\left\langle \overline{\mathbf{U}}^{\dagger (i_3,\cdots,i_d)}\overline{\mathbf{U}}^{(i_3,\cdots,i_d)}\right\rangle + \left\langle \overline{\mathbf{V}}^{\dagger (i_3,\cdots,i_d)}\overline{\mathbf{V}}^{(i_3,\cdots,i_d)}\right\rangle\right)_{rr}}{2}\right) \notag \\
  &\left. \qquad\qquad\qquad\qquad\qquad\qquad\qquad \times \bar{\lambda}_r^{(i_3,\cdots,i_d)}\right] + \text{const}. \label{ln q(lambda)}
\end{align}

From (\ref{ln q(lambda)}), we can recognize that the variational posterior of $q(\overline{\boldsymbol{\lambda}})$ follows the product of Gamma distributions
\begin{equation}
  q(\overline{\boldsymbol{\lambda}})   = \prod_{r i_3 \cdots i_d} \operatorname{Ga}\left(\bar{\lambda}_r^{(i_3,\cdots,i_d)} \mid \overline{a}_{r}^{\lambda (i_3, \cdots, i_d)},\overline{b}_{r}^{\lambda (i_3, \cdots, i_d)}\right), \label{updata lambda}
\end{equation}
where
\begin{align}
  &\overline{a}_{r}^{\lambda (i_3, \cdots, i_d)}=\overline{a}_0^{\lambda} +\frac{I_1+I_2}{2} , \\
  &\overline{b}_{r}^{\lambda (i_3, \cdots, i_d)}=\overline{\, b}_0^{\lambda} +\notag \\
  &\frac{\left(\left\langle \overline{\mathbf{U}}^{\dagger (i_3,\cdots,i_d)}\overline{\mathbf{U}}^{(i_3,\cdots,i_d)}\right\rangle + \left\langle \overline{\mathbf{V}}^{\dagger (i_3,\cdots,i_d)}\overline{\mathbf{V}}^{(i_3,\cdots,i_d)}\right\rangle\right)_{rr}}{2}.
\end{align}

The involved expectation terms $\left\langle \overline{\mathbf{U}}^{\dagger (i_3,\cdots,i_d)}\overline{\mathbf{U}}^{(i_3,\cdots,i_d)}\right\rangle$ and $\left\langle \overline{\mathbf{V}}^{\dagger (i_3,\cdots,i_d)}\overline{\mathbf{V}}^{(i_3,\cdots,i_d)}\right\rangle$ have been given in (\ref{expUU}) and (\ref{expVV}), respectively, and the posterior mean of $\bar{\lambda}_r^{(i_3,\cdots,i_d)}$ is given by $\left\langle \bar{\lambda}_r^{(i_3,\cdots,i_d)} \right\rangle =\overline{a}_{r}^{\lambda (i_3, \cdots, i_d)}/ \overline{b}_{r}^{\lambda (i_3, \cdots, i_d)}$.

\textit{3) Variational Posterior of $\mathcal{S}$.} To compute the variational posterior of $\mathcal{S}$, we use the log of the joint distribution in the original domain (\ref{original log}) as $\mathcal{S}$ is specified in the original domain. With $\boldsymbol{\Theta}_{j} = \mathcal{S}$ and apply (\ref{VB optimal}), we have   
\begin{align}
  &\ln q(\mathcal{S}) =  \notag \\
  & \sum_{i_1 \cdots i_d} \left[-\frac{\left\langle \tau \right\rangle  + \left\langle \beta_{i_1,\cdots, i_d}\right\rangle }{2} \mathcal{S}^2_{i_1, \cdots, i_d}  + \left\langle \tau  \right\rangle \mathcal{Z}_{i_1,\cdots, i_d}\mathcal{S}_{i_1, \cdots, i_d} \right]\notag \\
  &+ \text{const},  \label{ln q(S)}
\end{align}
where $\mathcal{Z}_{i_1,\cdots, i_d} = \mathcal{Y}_{i_1, \cdots, i_d} - L^{-1}\left(\overline{\mathcal{U}} \triangle \overline{\mathcal{V}}^{\dagger} \right)_{i_1, \cdots, i_d}$.

From (\ref{ln q(S)}), we see that $q(\mathcal{S})$ follows the product of Gaussian distributions
\begin{equation}
  q(\mathcal{S}) = \prod_{i_1  \cdots i_d}\mathcal{N}\left(\mathcal{S}_{i_1,\cdots,i_d} \mid \left \langle \mathcal{S}_{i_1,\cdots,i_d}\right \rangle , \sigma_{i_1,\cdots, i_d}^{2}\right), \label{update S}
\end{equation}
where
\begin{align}
\left \langle \mathcal{S}_{i_1,\cdots,i_d}\right \rangle &=  \left \langle\tau\right \rangle \left(\left \langle\beta_{i_1,\cdots, i_d}\right \rangle + \left \langle \tau \right \rangle\right)^{-1} \mathcal{Z}_{i_1,\cdots, i_d},\\
\sigma_{i_1,\cdots, i_d}^{2}&=  \left(\left \langle \beta_{i_1,\cdots, i_d}\right \rangle + \left \langle \tau \right \rangle\right)^{-1}.
\end{align}

\textit{4) Variational Posterior of $\boldsymbol{\beta}$}. Similarly, by applying (\ref{VB optimal}) with $\boldsymbol{\Theta}_{j} = \boldsymbol{\beta}$ we have
\begin{align}
  &\ln  q(\boldsymbol{\beta}) = \sum_{i_1 \cdots i_d}\left[\left(a_0^{\beta}+\frac{1}{2} -1\right) \ln \beta_{i_1, \cdots, i_d} \right. \notag \\
  & \left.\qquad \qquad \quad-\left(b_0^{\beta}+\frac{\left \langle \mathcal{S}^{2}_{i_1, \cdots, i_d}\right \rangle}{2}\right)\beta_{i_1, \cdots, i_d}\right]+ \text{const}. \label{ln q(beta)}
\end{align}

From (\ref{ln q(beta)}), we see that the variational posterior of $\boldsymbol{\beta}$ follows the product of Gamma distribution 
\begin{equation}
  q(\boldsymbol{\beta}) = \prod_{i_1  \cdots i_d}\mathrm{Ga}\left(\beta_{i_1, \cdots, i_d} \mid a_{i_1,\cdots, i_d}^{\beta}, b_{i_1,\cdots ,i_d}^{\beta}\right), \label{update beta}
\end{equation}
where
\begin{align}
  a_{i_1,\cdots, i_d}^{\beta}&=a_{0}^{\beta}+\frac{1}{2},\\
  b_{i_1,\cdots, i_d}^{\beta}&=b_{0}^{\beta}+\frac{1}{2}\left \langle \mathcal{S}^{2}_{i_1,\cdots, i_d}\right \rangle.
\end{align}
The expectation $\left \langle \mathcal{S}^{2}_{i_1, \cdots, i_d}\right \rangle$ can be computed by
\begin{equation}
  \left \langle \mathcal{S}^{2}_{i_1, \cdots, i_d}\right \rangle = \left \langle \mathcal{S}_{i_1, \cdots, i_d}\right \rangle ^{2} + \sigma_{i_1,\cdots, i_d}^{2}.
\end{equation}

\textit{5) Variational Posterior of $\tau$}. For simplicity of computation, we use the  log of the joint distribution in the transform domain (\ref{transform log}) and $\ln q(\tau)$ is given by
\begin{align}
  &\ln q(\tau) = \left( a_{0}^{\tau} +\frac{I_1 \cdots I_d}{2} -1 \right) \ln \tau - \tau \left(b_{0}^{\tau} + \right. \notag \\
  &\left.  \frac{\left\langle  \sum_{i_1 \cdots i_d} (\overline{\mathcal{Y}}_{i_1,\cdots, i_d} - \overline{\mathcal{U}}_{i_1,:,\cdots,i_d} \cdot \overline{\mathcal{V}}^{\dagger}_{i_2,:,\cdots,i_d} -\overline{\mathcal{S}}_{i_1, \cdots, i_d})^{2}\right\rangle}{2\varphi}  \right). \label{ln q(tau)}
\end{align}

From (\ref{ln q(tau)}), we see that the variational distribution of $\tau$ follows the Gamma distribution
\begin{equation}
  q(\tau)=\mathrm{Ga}\left(\tau \mid a^{\tau}, b^{\tau}\right), \label{update tau}
\end{equation}
where
\begin{align}
  &a^{\tau}=a_{0}^{\tau}+\frac{I_1\cdots I_d}{2},\\
  &b^{\tau}=b_{0}^{\tau} +  \notag \\
  &\frac{\left\langle  \sum_{i_1 \cdots i_d} (\overline{\mathcal{Y}}_{i_1, \cdots, i_d} -  \overline{\mathcal{U}}_{i_1,:,\cdots,i_d} \cdot \overline{\mathcal{V}}^{\dagger}_{i_2,:,\cdots,i_d}-\overline{\mathcal{S}}_{i_1, \cdots, i_d})^{2} \right\rangle}{2\varphi}, 
\end{align}
and the involved expectation term can be computed by 
\begin{align}
  & \left\langle  \sum_{i_1  \cdots i_d} (\overline{\mathcal{Y}}_{i_1, \cdots, i_d} -  \overline{\mathcal{U}}_{i_1,:,\cdots,i_d} \cdot \overline{\mathcal{V}}^{\dagger}_{i_2,:,\cdots,i_d}-\overline{\mathcal{S}}_{i_1, \cdots, i_d})^{2} \right\rangle =  \notag \\  
  &\sum_{i_1 \cdots i_d} \left(\overline{\mathcal{Y}}_{i_1, \cdots, i_d} -  \left\langle \overline{\mathcal{U}}_{i_1,:,\cdots,i_d}\right\rangle  \left\langle \overline{\mathcal{V}}_{i_2,:,\cdots,i_d}\right\rangle^{\dagger} -\left\langle\overline{\mathcal{S}}_{i_1, \cdots, i_d}\right\rangle \right)^{2} \notag\\
  &\qquad + I_1I_2 \sum_{i_3 \cdots i_d} \operatorname{tr}\left( \boldsymbol{\Sigma}_{\overline{v}}^{{(i_3,\cdots,i_d})}\boldsymbol{\Sigma}_{\overline{u}}^{(i_3,\cdots,i_d)}\right) \notag \\
  &\qquad +I_1\sum_{i_3 \cdots i_d} \operatorname{tr} \left( \boldsymbol{\Sigma}_{\overline{u}}^{(i_3,\cdots,i_d)}\left\langle\overline{\mathbf{V}}^{(i_3,\cdots,i_d)}  \right\rangle^{\dagger} \left\langle \overline{\mathbf{V}}^{(i_3,\cdots,i_d)}  \right\rangle\right) \notag \\
  &\qquad + I_2 \sum_{i_3\cdots i_d} \operatorname{tr}\left( \boldsymbol{\Sigma}_{\overline{v}}^{{(i_3,\cdots,i_d})} \left\langle\overline{\mathbf{U}}^{(i_3,\cdots,i_d)}  \right\rangle^{\dagger} \left\langle \overline{\mathbf{U}}^{(i_3,\cdots,i_d)}  \right\rangle\right)\notag \\
  &\qquad + \varphi \sum_{i_1\cdots i_d} \sigma_{i_1\cdots i_d}^{2}.  \label{exp fit}
\end{align}

\textit{Reﬁnement Trick for Model Training}. As stated in literature, due to the over-strength of the regularization term (such as $\operatorname{diag}(\left\langle\overline{\boldsymbol{\lambda}}^{(i_3,\cdots,i_d)}\right\rangle)$  in (\ref{update sigma u}) and (\ref{update sigma v}) in this paper), the model may be premature. For LMH-BRTF, the model tends to prune most factors before ﬁtting the input data, leading to poor performance of the model. By contrast, a more reasonable model should first fit the observations and not prune until good fitness has been achieved. To this end, inspired by \cite{BTRTF}, we use a reﬁnement trick to gradually enhance the effect of the regularization term. Specifically, we modify the updating equations (\ref{update sigma u}) and (\ref{update sigma v}) as 
\begin{align}
  \boldsymbol{\Sigma}_{\overline{u}}^{(i_3,\cdots,i_d)}&= \left(\frac{\left\langle\tau\right\rangle}{\varphi}\left\langle\overline{\mathbf{V}}^{\dagger (i_3,\cdots,i_d)}\overline{\mathbf{V}}^{(i_3,\cdots,i_d)}  \right\rangle \right. \notag \\
  & \left. \qquad \qquad \quad+ \frac{Fit}{\gamma} \operatorname{diag}(\left\langle\overline{\boldsymbol{\lambda}}^{(i_3,\cdots,i_d)}\right\rangle)\right)^{-1}, \label{modi update sigma u} \\
  \boldsymbol{\Sigma}_{\overline{v}}^{{(i_3,\cdots,i_d})}&= \left(\frac{\left\langle\tau\right\rangle}{\varphi}\left\langle\overline{\mathbf{U}}^{\dagger (i_3,\cdots,i_d)}\overline{\mathbf{U}}^{(i_3,\cdots,i_d)}  \right\rangle \right.\notag \\
  &\left.  \qquad \qquad \quad + \frac{Fit}{\gamma}\operatorname{diag}(\left\langle\overline{\boldsymbol{\lambda}}^{(i_3,\cdots,i_d)}\right\rangle)\right)^{-1}, \label{modi update sigma v}
\end{align}
respectively, where $ Fit =1- \sqrt{ \langle \| \overline{\mathcal{Y}}-\overline{\mathcal{U}} \triangle \overline{\mathcal{V}}^{\dagger}- \overline{\mathcal{S} }  \|_{F}^{2} \rangle}  /\|\overline{\mathcal{Y}}\|_{F}$ characterizes the goodness of fit to the observed tensor of the model and $\langle \| \overline{\mathcal{Y}}-\overline{\mathcal{U}} \triangle \overline{\mathcal{V}}^{\dagger}- \overline{\mathcal{S} }  \|_{F}^{2} \rangle$ can be obtained by (\ref{exp fit}). The refinement factor $\gamma >0$ can adjust the strength of the regularization term $\operatorname{diag}(\left\langle\overline{\boldsymbol{\lambda}}^{(i_3,\cdots,i_d)}\right\rangle)$.

In the early stage of model training, the model does not fit the observed tensor $\mathcal{Y}$ well, and therefore gets a small $Fit$. At this stage, the regularization term does not have much effect and thus no columns in $\{\overline{\mathbf{U}}^{k}\}$ and $\{\overline{\mathbf{V}}^{k}\}$ are pruned, allowing the model to focus more on fitting the observations. As the model gradually fits the observed tensor, $Fit$ will converge to 1 thus leading to a stronger effect of the regularization term. In this case, the modified updating equations return to the original ones (\ref{update sigma u}) and (\ref{update sigma v}) when $\gamma = 1$. In general, the value of $\gamma$ can be tuned for different tasks. However, we find that fixing $\gamma = \varphi$ is sufficient to achieve good performance in most situations. Therefore, we set $\gamma = \varphi$ in all of our experiments unless otherwise specified.

\subsection{Initialization}

Well-chosen initialization parameters are very important for LMH-BRTF, as the variational inference methods generally converge to local optima. We initialize the top-level parameters $ \overline{a}_{0}^{\lambda},\overline{\, b}_{0}^{\lambda}, a_{0}^{\beta}, b_{0}^{\beta}, a_{0}^{\tau} , b_{0}^{\tau}$ to $10^{-6}$ so as to introduce non-informative priors. The factor tensors $ \langle \overline{\mathcal{U}}  \rangle$ and  $\langle \overline{\mathcal{V}}  \rangle$ are initialized according to Lemma 1, given by $ \langle \overline{\mathcal{U}}  \rangle = \overline{\mathcal{U}_{0}} \triangle \overline{\mathcal{S}_0}^{\frac{1}{2}}$ and $ \langle \overline{\mathcal{V}}  \rangle = \overline{\mathcal{V}_{0}} \triangle \overline{\mathcal{S}_0}^{\frac{1}{2}}$, respectively, where $\overline{\mathcal{U}_{0}},\overline{\mathcal{S}_{0}},\overline{\mathcal{V}_{0}}$ are from $\overline{\mathcal{Y}} = \overline{\mathcal{U}_{0}} \triangle \overline{\mathcal{S}_{0}} \triangle \overline{\mathcal{V}_{0}}^{\dagger}$. We set the covariance metrics $\boldsymbol{\Sigma}_{\overline{u}}^{(i_3,\cdots,i_d)}$ and $\boldsymbol{\Sigma}_{\overline{v}}^{{(i_3,\cdots,i_d})}$ to the scaled identity metrix $\varphi \mathbf{I}$, where $\varphi$ is determined by the invertible linear
transform $L$. We set $ \langle \beta_{i_1,\cdots, i_d}\rangle = 1/\sigma_0^{2}$ and the elements in the sparse component $\langle \mathcal{S}_{i_1,\cdots,i_d} \rangle$ are sampled from the uniform distribution $\mathcal{U}(0,\sigma_0)$, where $\sigma_0$ is a task-specific constant and acts as the initial variance of  $\langle \mathcal{S}_{i_1,\cdots,i_d} \rangle$. The precision $\bar{\lambda}_r^{(i_3,\cdots,i_d)}$ of the $r$-th columns in $\overline{\mathbf{U}}^{(i_3,\cdots,i_d)}$ and $\overline{\mathbf{V}}^{(i_3,\cdots,i_d)}$ is set to $1/{\varphi}$. The noise precision $\left\langle \tau\right\rangle$ is set to $a_0^{\tau}/b_0^{\tau}=1$. 

The variational inference algorithm for LMH-BRTF is summarized in Algorithm \ref{alg1}.

\section{Experiments}

In this section, we conduct numerical experiments on both synthetic and real-world datasets to demonstrate the effectiveness of LMH-BRTF. Specifically, we first demonstrate the ability of LMH-BRTF to estimate the low-multi-rank tensor from the high-order corrupted observation tensor and automatically determine the multi-rank of the high-order tensor. We then apply LMH-BRTF to color video denoising and light-field images denoising. All the experiments are implemented on Matlab R2021b with an AMD Ryzen 5 3600 6-Core CPU.

\subsection{Synthetic Experiments}

\begin{algorithm}[tbp]
  \caption{Variational Inference for LMH-BRTF} 
  \label{alg1}
  \begin{algorithmic}[1]
  \REQUIRE The observed corrupted tensor $\mathcal{Y} \in \mathbf{R}^{I_1 \times \cdots \times I_d}$, and the initialized multi-rank $ \operatorname{rank}_{m}(\hat{\mathcal{X}}^{0}) \in \mathbf{R}^{I_3 \times \cdots \times I_d}$.
  \ENSURE The low-multi-rank component $\mathcal{X} \in \mathbf{R}^{I_1 \times \cdots \times I_d}$.

  \STATE Initialize $\overline{\mathcal{U}}$, $\overline{\mathcal{V}}$, $\overline{\boldsymbol{\lambda}}$, $\mathcal{S}$, $\boldsymbol{\beta}$, $\tau$.
  \REPEAT
  \STATE Update the posterior $q(\overline{\mathcal{U}})$ by (\ref{update U}) and (\ref{modi update sigma u});
  \STATE Update the posterior $q(\overline{\mathcal{V}})$ by (\ref{update V}) and (\ref{modi update sigma v});
  \STATE Update the posterior $q(\overline{\boldsymbol{\lambda}})$ by (\ref{updata lambda});
  \STATE Update the posterior $q(\mathcal{S})$ by (\ref{update S});
  \STATE Update the posterior $q(\boldsymbol{\beta})$ by (\ref{update beta}); 
  \STATE Update the posterior $q(\tau)$ by (\ref{update tau});
  \STATE Remove the unnecessary columns in $\{\overline{\mathbf{U}}^{k}\}$ and $\{\overline{\mathbf{V}}^{k}\}$ to achieve automatic multi-rank determination;
  \UNTIL convergence.
  \RETURN The low-multi-rank component $\mathcal{X}$.
  \end{algorithmic}
\end{algorithm}

We first demonstrate that LMH-BRTF can effectively infer the low-multi-rank tensor from the high-order corrupted observation tensor, and can automatically determine the multi-rank of the high-order tensor on the synthetic datasets. We generate an order-$d$ corrupted tensor $\mathcal{Y} \in \mathbf{R}^{I \times I \times I_3 \times \cdots \times I_d}$ as follows: Firstly, we independently sample the elements in the two factor tensors $\mathcal{U} \in \mathbf{R}^{I \times R \times \cdots \times I_d }$ and $\mathcal{V} \in \mathbf{R}^{I \times R \times \cdots \times I_d }$ from the standard Gaussian distribution $\mathcal{N}(0,1)$. Then, we construct the groundtruth low-rank component $\mathcal{X}_{gt}$ by the invertible linear transform $L$ based t-product, i.e., $\mathcal{X}_{gt} \in \mathbf{R}^{I \times I \times I_3\times \cdots \times I_d} = \mathcal{U} *_{L} \mathcal{V}$, and further truncate it by the order-$d$ t-SVD so that $\operatorname{Rank}_{\mathrm{m}}\left(\mathcal{X}_{g t}\right)=\left(R_{g t}^{1}, \ldots, R_{g t}^{J}\right)$, where $J = I_3 \times \cdots \times I_d$. For brevity, we limit the invertible linear transform $L$ to the DFT transformation. In addition, to construct the sparse component $\mathcal{S} \in \mathbf{R}^{I \times I \times \cdots \times I_d}$, we randomly set $\rho \%$ of its $I^{2} \times I_3 \times \cdots \times I_d$ elements to be non-zero, whose values are drawn from the uniform distribution ranged $[-10,10]$. For each element in the noise component $\mathcal{E} \in \mathbf{R}^{I \times I \times \cdots \times I_d}$, we sample independently from the Gaussian distribution $\mathcal{N}(0, \sigma^2)$, where $\sigma^2 = 10^{-4}$ corresponds to the small noise scenario, and $\sigma^2 = 10^{-1}$ corresponds to the large noise scenario. Finally, the observed corrupted tensor is constructed by $\mathcal{Y} = \mathcal{X}_{gt}+\mathcal{S}+ \mathcal{E}$.

In this experiment, we compare LMH-BRTF with its baseline method, BTRTF \cite{BTRTF}, which is only applicable to the third-order tensors. For BTRTF, we use the suggested parameters in \cite{BTRTF}. For LMH-BRTF, we initialize the sparse variance $\sigma_{0}^{2}=1$ and the reﬁnement factor $\gamma=1$, so that these parameters do not have an impact on the model learning. We set the initial multi-rank to $ \operatorname{rank}_{m}(\hat{\mathcal{X}}^{0}) = 0.5I \times \operatorname{ones}(I_3 \times \cdots \times I_d,1)$. The convergence criterion is set to $tol=\frac{\left\|\hat{\mathcal{X}}^{t}-\hat{\mathcal{X}}^{t-1}\right\|_{F}}{\left\|\hat{\mathcal{X}}^{t-1}\right\|_{F}}<10^{-6}$, where $\hat{\mathcal{X}}^{t}$ is the estimated low-rank tensor in the $t$-th iteration.

To evaluate the performance of the model in automatically determining the tensor multi-rank and accurately estimating the low-rank component, we use $R_{err}=\sum_{k=1}^{I_{3} \cdots I_{d}} \frac{\left|\hat{R}^{(k)}-R_{g t}^{(k)}\right|}{I_{3} \cdots I_{d}}$ and $\mathcal{X}_{err} = \frac{\left\|\hat{\mathcal{X}}-\mathcal{X}_{g t}\right\|_{F}}{\left\|\mathcal{X}_{g t}\right\|_{F}}$, respectively. The experimental results are shown in Table \ref{tab:syn1} and Table \ref{tab:syn2} where we set $I=100$ and $R=10$. Table \ref{tab:syn1} summarizes the results corresponding to the order-3 tensors, and Table \ref{tab:syn2} summarizes the results corresponding to the order-4 and order-5 tensors.

\begin{table}[]
  \centering
  \caption{Multi-rank Determination and Low-multi-rank Tensor Estimation Results of BTRTF and LMH-BRTF on the Order-3 Synthetic Datasets}
  \label{tab:syn1}
  \begin{tabular}{cccccc}
  \hline\hline
  \multicolumn{6}{c}{Observed corrupted tensor $\mathcal{Y} \in \mathbf{R}^{100 \times 100 \times 100}$}                                                                     \\ \hline
  \multicolumn{6}{c}{$\operatorname{Rank}_{\mathrm{m}}\left(\mathcal{X}_{g t}\right)=\{R, \overbrace{0.5 R, \ldots, 0.5 R}^{20}, \overbrace{R, \ldots, R}^{59}, \overbrace{0.5 R, \ldots, 0.5 R}^{20}\}$} \\ \hline
  \multicolumn{2}{c|}{Method}                                                 & \multicolumn{2}{c|}{BTRTF}                              & \multicolumn{2}{c}{LMH-BRTF}       \\ \hline
  \multicolumn{1}{c|}{$\rho$} &
    \multicolumn{1}{c|}{$\sigma^{2}$} &
    \multicolumn{1}{c|}{$R_{err}$} &
    \multicolumn{1}{c|}{$\mathcal{X}_{err}$} &
    \multicolumn{1}{c|}{$R_{err}$} &
    $\mathcal{X}_{err}$ \\ \hline
  \multicolumn{1}{c|}{\multirow{2}{*}{5\%}}  & \multicolumn{1}{c|}{$10^{-4}$} & \multicolumn{1}{c|}{0} & \multicolumn{1}{c|}{1.957e-04} & \multicolumn{1}{c|}{0} & 1.957e-04 \\ \cline{2-6} 
  \multicolumn{1}{c|}{}                      & \multicolumn{1}{c|}{$10^{-1}$} & \multicolumn{1}{c|}{0} & \multicolumn{1}{c|}{6.760e-03} & \multicolumn{1}{c|}{0} & 6.760e-03 \\ \hline
  \multicolumn{1}{c|}{\multirow{2}{*}{10\%}} & \multicolumn{1}{c|}{$10^{-4}$} & \multicolumn{1}{c|}{0} & \multicolumn{1}{c|}{2.010e-04} & \multicolumn{1}{c|}{0} & 2.010e-04 \\ \cline{2-6} 
  \multicolumn{1}{c|}{}                      & \multicolumn{1}{c|}{$10^{-1}$} & \multicolumn{1}{c|}{0} & \multicolumn{1}{c|}{6.970e-03} & \multicolumn{1}{c|}{0} & 6.970e-03 \\ \hline
  \multicolumn{1}{c|}{\multirow{2}{*}{20\%}} & \multicolumn{1}{c|}{$10^{-4}$} & \multicolumn{1}{c|}{0} & \multicolumn{1}{c|}{2.154e-04} & \multicolumn{1}{c|}{0} & 2.154e-04 \\ \cline{2-6} 
  \multicolumn{1}{c|}{}                      & \multicolumn{1}{c|}{$10^{-1}$} & \multicolumn{1}{c|}{0} & \multicolumn{1}{c|}{7.515e-03} & \multicolumn{1}{c|}{0} & 7.515e-03 \\ \hline\hline
  \end{tabular}%
\end{table}

As shown in Table \ref{tab:syn1}, we note that LMH-BRTF degenerates to BTRTF when $d=3$. This is because the order-$d$ t-SVD is perfectly compatible with the third-order case. In this case, both LMH-BRTF and BTRTF can effectively determine the tensor multi-rank and estimate the low-multi-rank component. However, in the higher-order cases, as shown in table \ref{tab:syn2} for $d=4,5$, BTRTF can no longer be effective. This is because BTRTF is only applicable to order-3 tensors and in order to deal with the higher-order case, BTRTF has to reshape the higher-order tensors into third-order ones. In this experiment, we merge the 3 to $d=4,5$ modes for BTRTF. However, this operation inevitably breaks the correlations among different dimensions of the original tensor, and consequently the corresponding performance in multi-rank determination and low-multi-rank tensor estimation is poor. In contrast, LMH-BRTF can still effectively capture the low-rank structure in the higher-order tensors, thus achieving good performance in multi-rank determination and low-multi-rank tensor estimation.

These results demonstrate that LMH-BRTF can effectively infer the low-multi-rank component from the high-order corrupted observation tensor and determine the tensor multi-rank.

\begin{table}[]
  \caption{Multi-rank Determination and Low-multi-rank Tensor Estimation Results of BTRTF and LMH-BRTF on the Order-4 and Order-5 Synthetic Datasets}
  \label{tab:syn2}
  \resizebox{\columnwidth}{!}{%
  \begin{tabular}{cccccc}
  \hline\hline
  \multicolumn{6}{c}{Observed corrupted tensor $\mathcal{Y} \in \mathbf{R}^{100 \times 100 \times 10 \times 10}$} \\ \hline
  \multicolumn{6}{c}{$\operatorname{Rank}_{\mathrm{m}}\left(\mathcal{X}_{g t}\right)=\{R,\overbrace{0.5 R, \ldots, 0.5 R}^{9}, \overbrace{R, \ldots, R}^{20} ,  \overbrace{0.5 R, \ldots, 0.5 R}^{50}, \overbrace{R, \ldots, R}^{20}\}$} \\ \hline
  \multicolumn{2}{c|}{Method} &
    \multicolumn{2}{c|}{BTRTF} &
    \multicolumn{2}{c}{LMH-BRTF} \\ \hline
  \multicolumn{1}{c|}{$\rho$} &
    \multicolumn{1}{c|}{$\sigma^{2}$} &
    \multicolumn{1}{c|}{$R_{err}$} &
    \multicolumn{1}{c|}{$\mathcal{X}_{err}$} &
    \multicolumn{1}{c|}{$R_{err}$} &
    $\mathcal{X}_{err}$ \\ \hline
  \multicolumn{1}{c|}{\multirow{2}{*}{5\%}} &
    \multicolumn{1}{c|}{$10^{-4}$} &
    \multicolumn{1}{c|}{21.14} &
    \multicolumn{1}{c|}{1.550e-1} &
    \multicolumn{1}{c|}{0} &
    1.905e-4 \\ \cline{2-6} 
  \multicolumn{1}{c|}{} &
    \multicolumn{1}{c|}{$10^{-1}$} &
    \multicolumn{1}{c|}{42.95} &
    \multicolumn{1}{c|}{1.600e-1} &
    \multicolumn{1}{c|}{0} &
    6.605e-3 \\ \hline
  \multicolumn{1}{c|}{\multirow{2}{*}{10\%}} &
    \multicolumn{1}{c|}{$10^{-4}$} &
    \multicolumn{1}{c|}{19.48} &
    \multicolumn{1}{c|}{1.682e-1} &
    \multicolumn{1}{c|}{0} &
    1.955e-4 \\ \cline{2-6} 
  \multicolumn{1}{c|}{} &
    \multicolumn{1}{c|}{$10^{-1}$} &
    \multicolumn{1}{c|}{42.95} &
    \multicolumn{1}{c|}{1.712e-1} &
    \multicolumn{1}{c|}{0} &
    6.798e-3 \\ \hline
  \multicolumn{1}{c|}{\multirow{2}{*}{20\%}} &
    \multicolumn{1}{c|}{$10^{-4}$} &
    \multicolumn{1}{c|}{42.95} &
    \multicolumn{1}{c|}{1.899e-1} &
    \multicolumn{1}{c|}{0} &
    2.088e-4 \\ \cline{2-6} 
  \multicolumn{1}{c|}{} &
    \multicolumn{1}{c|}{$10^{-1}$} &
    \multicolumn{1}{c|}{42.95} &
    \multicolumn{1}{c|}{1.914e-1} &
    \multicolumn{1}{c|}{0} &
    7.321e-3 \\ \hline
  \multicolumn{6}{l}{$\operatorname{Rank}_{\mathrm{m}}\left(\mathcal{X}_{g t}\right)=\{0.5R,\overbrace{R, \ldots, R}^{9}, \overbrace{0.5R, \ldots, 0.5R}^{20} ,  \overbrace{R, \ldots,  R}^{50}, \overbrace{0.5R, \ldots, 0.5R}^{20}\}$} \\ \hline
  \multicolumn{1}{c|}{\multirow{2}{*}{5\%}} &
    \multicolumn{1}{c|}{$10^{-4}$} &
    \multicolumn{1}{c|}{42.05} &
    \multicolumn{1}{c|}{1.790e-1} &
    \multicolumn{1}{c|}{0} &
    1.959e-4 \\ \cline{2-6} 
  \multicolumn{1}{c|}{} &
    \multicolumn{1}{c|}{$10^{-1}$} &
    \multicolumn{1}{c|}{42.05} &
    \multicolumn{1}{c|}{1.830e-1} &
    \multicolumn{1}{c|}{0} &
    6.749e-3 \\ \hline
  \multicolumn{1}{c|}{\multirow{2}{*}{10\%}} &
    \multicolumn{1}{c|}{$10^{-4}$} &
    \multicolumn{1}{c|}{42.05} &
    \multicolumn{1}{c|}{1.908e-1} &
    \multicolumn{1}{c|}{0} &
    2.006e-4 \\ \cline{2-6} 
  \multicolumn{1}{c|}{} &
    \multicolumn{1}{c|}{$10^{-1}$} &
    \multicolumn{1}{c|}{42.05} &
    \multicolumn{1}{c|}{1.944e-1} &
    \multicolumn{1}{c|}{0} &
    6.977e-3 \\ \hline
  \multicolumn{1}{c|}{\multirow{2}{*}{20\%}} &
    \multicolumn{1}{c|}{$10^{-4}$} &
    \multicolumn{1}{c|}{42.05} &
    \multicolumn{1}{c|}{2.135e-1} &
    \multicolumn{1}{c|}{0} &
    2.146e-4 \\ \cline{2-6} 
  \multicolumn{1}{c|}{} &
    \multicolumn{1}{c|}{$10^{-1}$} &
    \multicolumn{1}{c|}{42.05} &
    \multicolumn{1}{c|}{2.149e-1} &
    \multicolumn{1}{c|}{0} &
    7.492e-3 \\ \hline\hline
  \multicolumn{6}{c}{Observed corrupted tensor $\mathcal{Y} \in \mathbf{R}^{100 \times 100 \times 5 \times 5 \times 5}$} \\ \hline
  \multicolumn{6}{c}{$\operatorname{Rank}_{\mathrm{m}}\left(\mathcal{X}_{g t}\right)=\{R,\overbrace{0.5 R, \ldots, 0.5 R}^{24}, \overbrace{R, \ldots, R}^{25} ,  \overbrace{0.5 R, \ldots, 0.5 R}^{50}, \overbrace{R, \ldots, R}^{25}\}$} \\ \hline
  \multicolumn{2}{c|}{Method} &
    \multicolumn{2}{c|}{BTRTF} &
    \multicolumn{2}{c}{LMH-BRTF} \\ \hline
  \multicolumn{1}{c|}{$\rho$} &
    \multicolumn{1}{c|}{$\sigma^{2}$} &
    \multicolumn{1}{c|}{$R_{err}$} &
    \multicolumn{1}{c|}{$\mathcal{X}_{err}$} &
    \multicolumn{1}{c|}{$R_{err}$} &
    $\mathcal{X}_{err}$ \\ \hline
  \multicolumn{1}{c|}{\multirow{2}{*}{5\%}} &
    \multicolumn{1}{c|}{$10^{-4}$} &
    \multicolumn{1}{c|}{42.96} &
    \multicolumn{1}{c|}{3.023e-1} &
    \multicolumn{1}{c|}{0} &
    1.687e-4 \\ \cline{2-6} 
  \multicolumn{1}{c|}{} &
    \multicolumn{1}{c|}{$10^{-1}$} &
    \multicolumn{1}{c|}{42.96} &
    \multicolumn{1}{c|}{3.028e-1} &
    \multicolumn{1}{c|}{0} &
    5.900e-3 \\ \hline
  \multicolumn{1}{c|}{\multirow{2}{*}{10\%}} &
    \multicolumn{1}{c|}{$10^{-4}$} &
    \multicolumn{1}{c|}{42.96} &
    \multicolumn{1}{c|}{3.070e-1} &
    \multicolumn{1}{c|}{0} &
    1.726e-4 \\ \cline{2-6} 
  \multicolumn{1}{c|}{} &
    \multicolumn{1}{c|}{$10^{-1}$} &
    \multicolumn{1}{c|}{42.96} &
    \multicolumn{1}{c|}{3.070e-1} &
    \multicolumn{1}{c|}{0} &
    6.098e-3 \\ \hline
  \multicolumn{1}{c|}{\multirow{2}{*}{20\%}} &
    \multicolumn{1}{c|}{$10^{-4}$} &
    \multicolumn{1}{c|}{42.96} &
    \multicolumn{1}{c|}{3.164e-1} &
    \multicolumn{1}{c|}{0} &
    1.829e-4 \\ \cline{2-6} 
  \multicolumn{1}{c|}{} &
    \multicolumn{1}{c|}{$10^{-1}$} &
    \multicolumn{1}{c|}{42.96} &
    \multicolumn{1}{c|}{3.179e-1} &
    \multicolumn{1}{c|}{0} &
    6.545e-3 \\ \hline
  \multicolumn{6}{l}{$\operatorname{Rank}_{\mathrm{m}}\left(\mathcal{X}_{g t}\right)=\{0.5R,\overbrace{ R, \ldots,  R}^{24}, \overbrace{0.5R, \ldots, 0.5R}^{25} ,  \overbrace{ R, \ldots, R}^{50}, \overbrace{0.5R, \ldots, 0.5R}^{25}\}$} \\ \hline
  \multicolumn{1}{c|}{\multirow{2}{*}{5\%}} &
    \multicolumn{1}{c|}{$10^{-4}$} &
    \multicolumn{1}{c|}{42.04} &
    \multicolumn{1}{c|}{3.305e-1} &
    \multicolumn{1}{c|}{0} &
    1.729e-4 \\ \cline{2-6} 
  \multicolumn{1}{c|}{} &
    \multicolumn{1}{c|}{$10^{-1}$} &
    \multicolumn{1}{c|}{42.04} &
    \multicolumn{1}{c|}{3.307e-1} &
    \multicolumn{1}{c|}{0} &
    6.019e-3 \\ \hline
  \multicolumn{1}{c|}{\multirow{2}{*}{10\%}} &
    \multicolumn{1}{c|}{$10^{-4}$} &
    \multicolumn{1}{c|}{42.04} &
    \multicolumn{1}{c|}{3.347e-1} &
    \multicolumn{1}{c|}{0} &
    1.776e-4 \\ \cline{2-6} 
  \multicolumn{1}{c|}{} &
    \multicolumn{1}{c|}{$10^{-1}$} &
    \multicolumn{1}{c|}{42.04} &
    \multicolumn{1}{c|}{3.347e-1} &
    \multicolumn{1}{c|}{0} &
    6.235e-3 \\ \hline
  \multicolumn{1}{c|}{\multirow{2}{*}{20\%}} &
    \multicolumn{1}{c|}{$10^{-4}$} &
    \multicolumn{1}{c|}{42.04} &
    \multicolumn{1}{c|}{3.415e-1} &
    \multicolumn{1}{c|}{0} &
    1.883e-4 \\ \cline{2-6} 
  \multicolumn{1}{c|}{} &
    \multicolumn{1}{c|}{$10^{-1}$} &
    \multicolumn{1}{c|}{42.04} &
    \multicolumn{1}{c|}{3.422e-1} &
    \multicolumn{1}{c|}{0} &
    6.699e-3 \\ \hline\hline
  \end{tabular}%
  }
\end{table}

\begin{table}[]
  \centering
  \caption{Summary of Different types of Competing TRPCA Methods.}
  \label{tab:competing methods}
  \begin{tabular}{clcc}
  \hline\hline
  Type           & \multicolumn{2}{c}{Method}                          & Order-$d$ tensor \\ \hline
  CP              & BRTF\cite{BRTF}   & \multirow{2}{*}{KBR\cite{KBR}}   & $\surd $         \\
  Tucker          & SNN\cite{SNN}     &                                  & $\surd $         \\
  Order-3   t-SVD & \multicolumn{2}{c}{TNN\cite{TNN}, BTRTF\cite{BTRTF}} & $\times$         \\ \hline\hline
  \end{tabular}
\end{table}

\subsection{Application to Color Video Denoising}

In this subsection, we apply LMH-BRTF to color video denoising where the videos are corrupted by the noises. In this task, the noises are absorbed by the sparse component and the noise component and the clean videos can be approximated by the low-rank component.

In order to evaluate LMH-BRTF, we download 10 color videos from the derf website\footnote{https://media.xiph.org/video/derf/}, namely Coastguard, Flower, Husky, Mobile, Bus, City, Football, Foreman, Harbour, and Soccer, respectively. Further, due to computational limitations, we take only the first 30 frames of these videos, where each frame has a size of $288 \times 352 \times 3 $. Therefore, each video can be formulated as a fourth-order tensor sized $288 \times 352 \times 3 \times 30$. For each video tensor, we randomly corrupt $\rho=20\%$ of its elements to random values ranging from $[0, 255]$. Following the common settings, the elements of each video are then normalized to $[0,1]$. 
Furthermore, we corrupt all the elements of the video tensor by adding i.i.d. noise elements drawn from the Gaussian distribution $\mathcal{N}(0, \sigma^2)$, where we set $\sigma^{2} = 10^{-4}$.

\begin{table*}[tb]
  \centering
  \caption{Video Denoising Results in Terms of PSNR, SSIM, FSIM, and CPU Time with $\rho = 20\%$ and $\sigma^{2} = 10^{-4}$}
  \label{tab:tableExp1}
  \begin{tabular}{cccccccccc}
  \hline\hline
  Index               & Dataset                 & Metric  & Observed & BRTF            & SNN     & KBR      & TNN     & BTRTF   & LMH-BRTF        \\ \hline
  \multirow{4}{*}{1} & \multirow{4}{*}{Coastguard} & PSNR & 14.509 & 18.980  & 25.918 & 27.231 & 30.599 & 31.748 & \textbf{34.895} \\
                      &                          & SSIM    & 0.167    & 0.321           & 0.797   & 0.824    & 0.927   & 0.937   & \textbf{0.968}  \\
                      &                          & FSIM    & 0.623    & 0.529           & 0.887   & 0.902    & 0.952   & 0.959   & \textbf{0.979}  \\
                      &                          & Time(s) & ------   & \textbf{99.632} & 193.742 & 994.245  & 201.056 & 178.174 & 216.879         \\ \hline
  \multirow{4}{*}{2}  & \multirow{4}{*}{Flower}  & PSNR    & 18.624   & 13.657          & 20.522  & 21.550    & 23.621  & 24.582  & \textbf{27.479} \\
                      &                          & SSIM    & 0.248    & 0.497           & 0.580    & 0.758    & 0.848   & 0.893   & \textbf{0.930}   \\
                      &                          & FSIM    & 0.667    & 0.525           & 0.771   & 0.835    & 0.889   & 0.910   & \textbf{0.942}  \\
                      &                          & Time(s) & ------   & \textbf{56.601} & 189.158   & 1140.542 & 204.449 & 259.379 & 256.463         \\ \hline
  \multirow{4}{*}{3}  & \multirow{4}{*}{Husky}   & PSNR    & 13.442   & 13.521          & 19.570   & 17.642   & 22.138  & 21.921  & \textbf{25.161} \\
                      &                          & SSIM    & 0.413    & 0.115           & 0.714   & 0.595    & 0.862   & 0.860    & \textbf{0.924}  \\
                      &                          & FSIM    & 0.792    & 0.339           & 0.882   & 0.817    & 0.920   & 0.921   & \textbf{0.960}  \\
                      &                          & Time(s) & ------   & \textbf{50.062} & 163.546 & 1182.212 & 199.534 & 253.921 & 306.217         \\ \hline
  \multirow{4}{*}{4}  & \multirow{4}{*}{Mobile}  & PSNR    & 13.442   & 11.554          & 20.570   & 20.124   & 23.479  & 24.180   & \textbf{25.938} \\
                      &                          & SSIM    & 0.342    & 0.185           & 0.696   & 0.694    & 0.867   & 0.900     & \textbf{0.926}  \\
                      &                          & FSIM    & 0.697    & 0.465           & 0.852   & 0.844    & 0.922   & 0.936   & \textbf{0.953}  \\
                      &                          & Time(s) & ------   & \textbf{54.723} & 185.423 & 1255.522 & 188.007 & 253.427 & 313.674         \\ \hline
  \multirow{4}{*}{5}  & \multirow{4}{*}{Bus}     & PSNR    & 12.851   & 14.210           & 20.323  & 22.072   & 25.233   & 25.606  & \textbf{29.523} \\
                      &                          & SSIM    & 0.245    & 0.244           & 0.577   & 0.708    & 0.848   & 0.860    & \textbf{0.934}  \\
                      &                          & FSIM    & 0.663    & 0.473           & 0.807   & 0.833    & 0.908   & 0.911   & \textbf{0.959}  \\
                      &                          & Time(s) & ------   & \textbf{54.269} & 163.563  & 1094.468 & 184.328 & 249.399 & 294.003         \\ \hline
  \multirow{4}{*}{6}  & \multirow{4}{*}{City}    & PSNR    & 14.128   & 18.874          & 24.551  & 25.998   & 27.896  & 27.708  & \textbf{30.403} \\
                      &                          & SSIM    & 0.165    & 0.304           & 0.671   & 0.751    & 0.843   & 0.840    & \textbf{0.914}  \\
                      &                          & FSIM    & 0.601    & 0.449           & 0.850    & 0.869    & 0.908   & 0.901   & \textbf{0.946}  \\
                      &                          & Time(s) & ------   & \textbf{58.623} & 163.663 & 981.306  & 175.072 & 190.381 & 249.252         \\ \hline
  \multirow{4}{*}{7} & \multirow{4}{*}{Football}   & PSNR & 13.759 & 15.406 & 20.771 & 24.967 & 24.125 & 24.066 & \textbf{27.378} \\
                      &                          & SSIM    & 0.183    & 0.319           & 0.569   & 0.797    & 0.806   & 0.832   & \textbf{0.914}  \\
                      &                          & FSIM    & 0.604    & 0.522           & 0.792   & 0.897    & 0.881   & 0.892   & \textbf{0.941}  \\
                      &                          & Time(s) & ------   & \textbf{73.559} & 183.936 & 1239.631 & 182.168 & 245.261 & 313.805         \\ \hline
  \multirow{4}{*}{8}  & \multirow{4}{*}{Foreman} & PSNR    & 16.597   & 16.037          & 26.621  & 30.117   & 28.928  & 29.869  & \textbf{31.611} \\
                      &                          & SSIM    & 0.115    & 0.537           & 0.744   & 0.856    & 0.872   & 0.896   & \textbf{0.915}  \\
                      &                          & FSIM    & 0.507    & 0.663           & 0.885   & 0.925    & 0.927   & 0.932   & \textbf{0.946}  \\
                      &                          & Time(s) & ------   & \textbf{60.374} & 158.165 & 798.142  & 182.050  & 159.873 & 191.498         \\ \hline
  \multirow{4}{*}{9}  & \multirow{4}{*}{Harbour} & PSNR    & 15.977   & 15.054          & 24.624  & 26.251   & 27.463  & 28.633  & \textbf{32.268} \\
                      &                          & SSIM    & 0.299    & 0.193           & 0.845   & 0.882    & 0.924   & 0.935   & \textbf{0.969}  \\
                      &                          & FSIM    & 0.696    & 0.479           & 0.915   & 0.926    & 0.944   & 0.951   & \textbf{0.976}  \\
                      &                          & Time(s) & ------   & \textbf{55.292} & 177.307 & 1152.984 & 174.713 & 172.373 & 238.989         \\ \hline
  \multirow{4}{*}{10} & \multirow{4}{*}{Soccer}  & PSNR    & 14.491   & 16.907          & 21.031  & 27.346   & 27.242  & 27.968  & \textbf{30.040}  \\
                      &                          & SSIM    & 0.130     & 0.436           & 0.614   & 0.809    & 0.844   & 0.863   & \textbf{0.909}  \\
                      &                          & FSIM    & 0.562    & 0.591           & 0.794   & 0.886    & 0.896   & 0.906   & \textbf{0.938}  \\
                      &                          & Time(s) & ------   & \textbf{69.538} & 174.276 & 945.505  & 179.587 & 188.728 & 225.795         \\ \hline\hline
  \end{tabular}
\end{table*}

\begin{figure}[htbp]
	\centering
	\includegraphics[scale=1,width=3.4in]{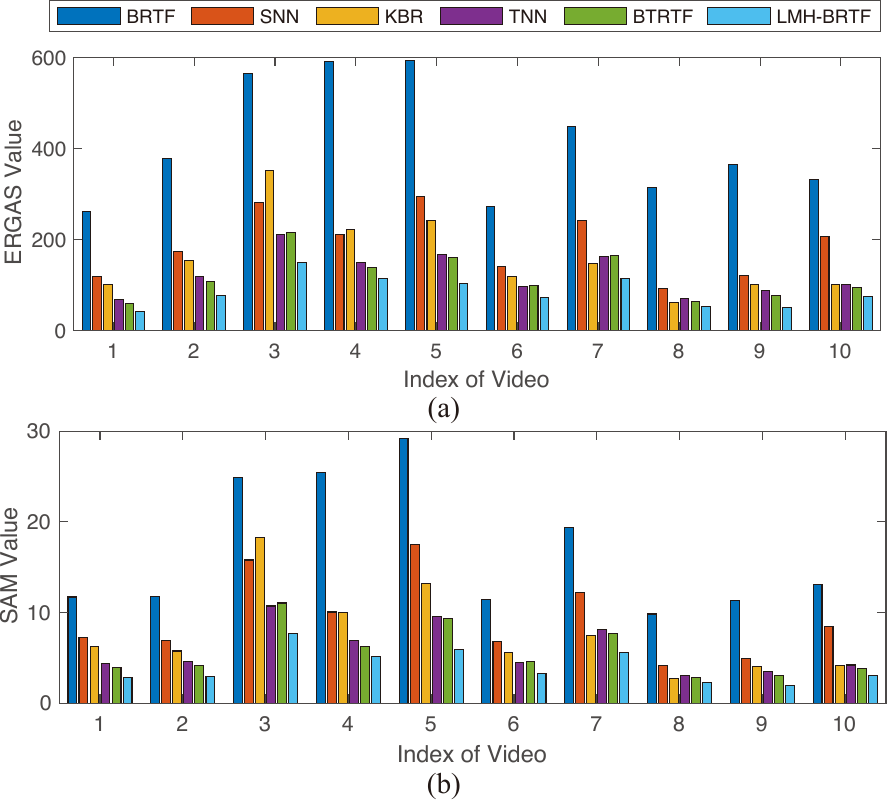}
	\caption{Comparison of color video denoising results on 10 videos. (a) Comparison with respect to ERGAS. (b) Comparison with respect to SAM.}
	\label{exp1_values}
\end{figure}

\begin{figure*}[htbp]
	\centering
	\includegraphics[scale=1,width=7in]{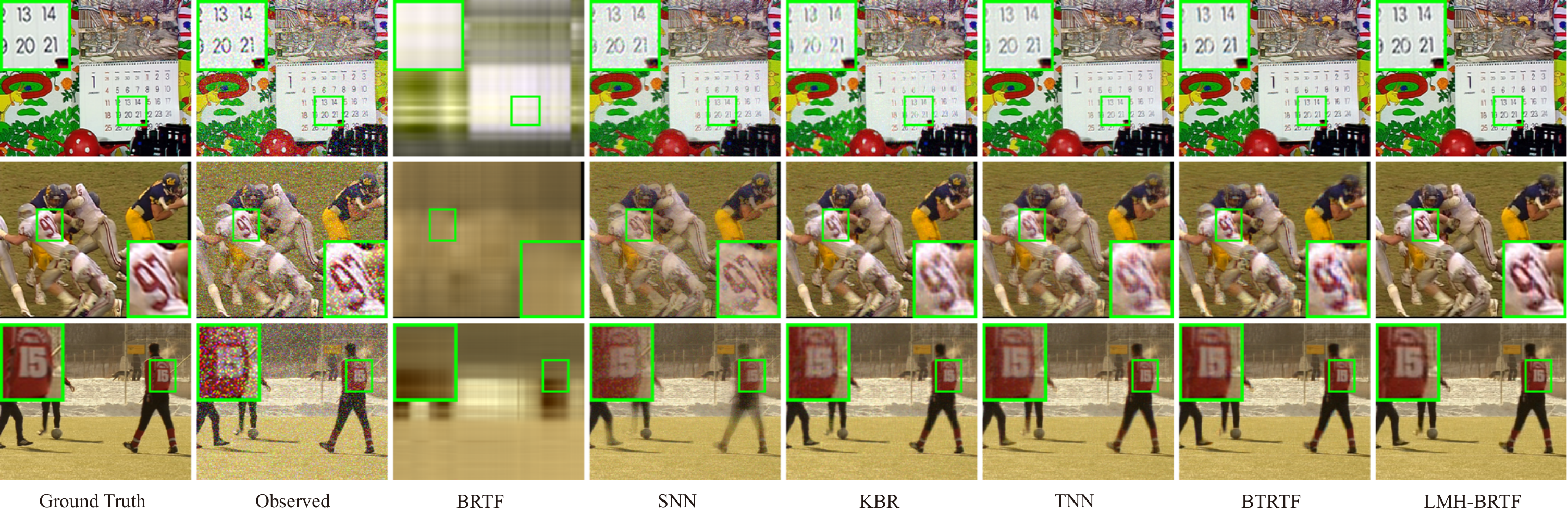}
	\caption{Color video denoising results on three representative videos with $\rho=20\%$ and $\sigma^2 = 10^{-4}$. The names of the videos from top to bottom are Mobile, Football, and Soccer, respectively.}
	\label{exp1_show}
\end{figure*}

We use the peak signal-to-noise ratio (PSNR\footnote{$\mathrm{PSNR}=10 \log _{10}\left(I_{1} \times I_{2} \times \cdots  \times I_{d}\|\mathcal{X}_{gt}\|_{\infty}^{2} /\|\hat{\mathcal{X}}-\mathcal{X}_{gt}\|_{F}^{2}\right)$}), the structural similarity (SSIM \cite{SSIM}), the feature similarity (FSIM \cite{FSIM}), the erreur relative globale adimensionnelle de synthese (ERGAS\cite{ERGAS}), and spectral angle mapper (SAM\cite{SAM}) to evaluate the denoising performance. PSNR and SSIM are two classical metrics in the field of image processing and computer vision, which are based on MSE and structural consistency, respectively, so as to measure the relationship between the target image and the reference image. In comparison, FSIM pays more attention to perceptual consistency. Larger values of these three metrics indicate more similar the target image is to the reference image and thus are better. For ERGAS, it measures ﬁdelity of the restored video based on the weighted sum of MSE in each frame and SAM determines the similarity between two videos by treating them as vectors in a space with dimensionality equal to the number of video frames and calculates the angle between them. For ERGAS and SAM, smaller values are better. We also report the CPU time to quantify the efficiency of each method.  

We compare the proposed method with several popular TRPCA methods, including the CP-based TRPCA method \cite{BRTF}, Tucker-based TRPCA method \cite{SNN}, TRPCA method combining both the CP and Tucker decomposition \cite{KBR}, and the order-3 t-SVD based methods\cite{TNN,BTRTF}. These competing methods are summarized in Table \ref{tab:competing methods}. For the CP-based and Tucker-based TRPCA methods, they can be directly applied to the order-$d$ tensors, whereas for the order-3 t-SVD based TRPCA methods, these algorithms can only reshape the higher-order tensors to third-order tensors. In this experiment, we merge the third mode with the fourth mode of the color videos so that the input tensor sized $288 \times 352 \times 90$. For LMH-BRTF, we specify the invertible linear transforms $L$ as the DFT transformation and the corresponding constant $\varphi = I_3 \times \cdots I_d$.

All the parameters in the competing methods are optimally tuned or set as the suggested values in the corresponding papers. Specifically, for BRTF, we set the initial rank to $\max(I_1,I_2,I_3,I_4)$; For SNN, we set the parameter $\boldsymbol{\lambda}=[\lambda_1,\lambda_2,\lambda_3,\lambda_4]$ as $[15,15,1.5,300]$; For KBR, we set $\beta = 2000$, $\gamma = 100\beta$, $\lambda=10$, $\mu = 250$, $\rho = 1.05$; For TNN, we set the balancing parameter $\lambda=1 / \sqrt{I_{3}\times I_4 \times \max \left(I_{1}, I_{2}\right)}$; For BTRTF, we set the initial multi-rank to $150 \times \operatorname{ones}(I_3 \times I_4,1)$, and the sparse variance $\sigma_{0}^{2} = 10^{-7}$; For LMH-BRTF, we initialize the multi-rank to $150 \times \operatorname{ones}(I_3 \times I_4,1)$ and set the convergence criterion to $tol < 10^{-4}$. We set the sparse variance $\sigma_{0}^{2}$ to $10^{-7}$ so that the values of the elements in the initial sparse component $\mathcal{S}_0$ are close to 0, thereby making the model tend to use the low-rank component rather than the sparse component to fit the observed data.

The color video denoising results on 10 video datasets are shown in Table \ref{tab:tableExp1}. It can be seen that LMH-BRTF achieves the best denoising performance among all the competitors in terms of PSNR, SSIM, and FSIM. Specifically, LMH-BRTF exceeds the second best method by around 2 to 4 dB with respect to PSNR. This can be attributed to the fact that LMH-BRTF is based on order-$d$ t-SVD, which is designed for higher-order tensors so that direct processing can be applied to them. In addition, the order-$d$ t-SVD can effectively capture the correlations among dimensions so that LMH-BRTF can have good denoising performance. In contrast, the denoising performance of the order-3 t-SVD based methods, TNN and BTRTF, is relatively poor. This is because these methods can only process the order-3 tensors. For higher-order tensors, one has to reshape them into order-3 tensors and then feed them into TNN or BTRTF. However, the reshape operation inevitably breaks the correlations among the dimensions of the original high-order tensor so that the denoising results of TNN and BRTRF are relatively poor. Another way for BTRTF and TNN to process the corrupted color videos is to process each video frame by frame. However, we get even poorer denoising results. This is because this method entirely destroys the correlations in the time dimension. Also, LMH-BRTF achieves better performance than the methods based on CP decomposition and Tucker decomposition, which is due to the fact that CP decomposition and Tucker decomposition fail to capture the complex low-rank structures in video data. We also compare the running time of these TRPCA methods. It turns out that BRTF is a more efficient method, but it fails to perform well as it can not accurately capture the correlation among the dimensions of the video tensors and tends to obtain an inaccurate low-rank component with the underestimated rank. As for LMH-BRTF, its running time is comparable to that of the TNN and BTRTF. The ERGAS and SAM values of these ten videos evaluated from LMH-BRTF and the five competing methods are shown in Fig. \ref{exp1_values}. These results also demonstrate the superiority of LMH-BRTF in color video denoising. Fig. \ref{exp1_show} further shows the representative visualization results in three videos, namely Mobile, Football, and Soccer, respectively. We see that LMH-BRTF can obtain videos with more details from corrupted videos and with colors that are more consistent with the ground-truth videos. For example, LMH-BRTF obtains digits more accurately and with more details from all the three corrupted videos. Additionally, the color of the digits obtained by LMH-BRTF is closer to that in the ground-truth video of Football.

\subsection{Application to Light-Field Images Denoising}

\begin{table*}[tb]
  \centering
  \caption{Light-Field Images Denoising Results in Terms of PSNR, SSIM, FSIM, and CPU Time with $\rho = 20\%$ and $\sigma^{2} = 10^{-4}$}
  \label{tab:tableExp3}
  \begin{tabular}{cccccccccc}
  \hline\hline
  Index & Dataset & Metric & Observed & \multicolumn{1}{c}{BRTF} & \multicolumn{1}{c}{SNN} & \multicolumn{1}{c}{KBR} & TNN & BTRTF & LMH-BRTF \\ \hline
  \multirow{4}{*}{1} & \multirow{4}{*}{Amethyst} & PSNR    & 12.445 & 14.055           & 20.431  & 18.532   & 27.500  & 28.361   & \textbf{29.585} \\
                     &                           & SSIM    & 0.081  & 0.613            & 0.436   & 0.370     & 0.881   & 0.921    & \textbf{0.936}  \\
                     &                           & FSIM    & 0.408  & 0.696            & 0.681   & 0.698    & 0.914   & 0.932    & \textbf{0.947}  \\
                     &                           & Time(s) & ------ & \textbf{150.687} & 631.403 & 1287.070  & 352.479 & 475.438  & 611.887         \\ \hline
  \multirow{4}{*}{2} & \multirow{4}{*}{Bunny}    & PSNR    & 12.888 & 11.841           & 20.253  & 18.431   & 27.378  & 28.738   & \textbf{30.157} \\
                     &                           & SSIM    & 0.102  & 0.571            & 0.383   & 0.420     & 0.889   & 0.929    & \textbf{0.943}  \\
                     &                           & FSIM    & 0.417  & 0.708            & 0.650    & 0.698    & 0.921   & 0.944    & \textbf{0.958}  \\
                     &                           & Time(s) & ------ & \textbf{150.885} & 623.725   & 1233.092 & 322.284 & 446.275  & 552.526         \\ \hline
  \multirow{4}{*}{3} & \multirow{4}{*}{Beans}    & PSNR    & 14.104 & 12.053           & 19.866  & 18.181   & 27.159  & 27.924   & \textbf{30.594} \\
                     &                           & SSIM    & 0.127  & 0.484            & 0.370    & 0.541    & 0.883   & 0.917    & \textbf{0.940}  \\
                     &                           & FSIM    & 0.444  & 0.721            & 0.660    & 0.705    & 0.920   & 0.933    & \textbf{0.953}  \\
                     &                           & Time(s) & ------ & \textbf{101.158} & 484.205 & 1090.165 & 376.120 & 1004.599 & 862.645         \\ \hline
  \multirow{4}{*}{4} & \multirow{4}{*}{Bracelet} & PSNR    & 14.199 & 12.666           & 22.307  & 18.616   & 28.519  & 30.117   & \textbf{33.440} \\
                     &                           & SSIM    & 0.117  & 0.423            & 0.525   & 0.502    & 0.903   & 0.941    & \textbf{0.961}  \\
                     &                           & FSIM    & 0.439  & 0.706            & 0.746   & 0.732    & 0.939   & 0.955    & \textbf{0.973}  \\
                     &                           & Time(s) & ------ & \textbf{113.726} & 628.441 & 1089.964 & 369.953 & 837.029  & 695.004         \\ \hline
  \multirow{4}{*}{5} & \multirow{4}{*}{Chess}    & PSNR    & 13.796 & 14.955           & 26.503  & 18.751   & 31.635  & 33.878   & \textbf{36.177} \\
                     &                           & SSIM    & 0.088  & 0.511            & 0.712   & 0.537    & 0.921   & 0.959    & \textbf{0.968}  \\
                     &                           & FSIM    & 0.415  & 0.721            & 0.847   & 0.746    & 0.955   & 0.970    & \textbf{0.978}  \\
                     &                           & Time(s) & ------ & \textbf{112.559} & 637.349 & 1089.744 & 368.993 & 610.272  & 617.313         \\ \hline
  \multirow{4}{*}{6} & \multirow{4}{*}{Portrait} & PSNR    & 14.849 & 10.743           & 21.713  & 18.183   & 28.226  & 30.183   & \textbf{33.390} \\
                     &                           & SSIM    & 0.154  & 0.382            & 0.469   & 0.522    & 0.920   & 0.950    & \textbf{0.969}  \\
                     &                           & FSIM    & 0.500  & 0.643            & 0.726   & 0.700      & 0.944   & 0.961    & \textbf{0.978}  \\
                     &                           & Time(s) & ------ & \textbf{99.555}  & 594.646 & 967.387  & 378.419 & 691.658  & 631.088         \\ \hline\hline
  \end{tabular}
\end{table*}

In this subsection, we demonstrate the denoising performance of LMH-BRTF on fifth-order light-field images data. We use six light-field images datasets from the Stanford Light-Field Archive\footnote{http://lightfield.stanford.edu/lfs.html}, namely, Amethyst, Bunny, Beans, Bracelet, Chess, and Portrait, respectively. The original size of these data is $3888 \times 2592 \times 3 \times 17 \times 17$. In order to reduce the computational load, we downsample these data to $216 \times 324 \times 3 \times 9 \times 9$. For each light-field images tensor, we randomly corrupt $\rho=20\%$ of its elements to random values ranging from $[0, 255]$. Following the common settings, the elements of each light-field images tensor are then normalized to $[0,1]$. Furthermore, we corrupt all the elements of the light-field images tensor by adding i.i.d. noise elements drawn from the Gaussian distribution $\mathcal{N}(0, \sigma^2)$, where we set $\sigma^{2} = 10^{-4}$.

\begin{figure}[htbp]
	\centering
	\includegraphics[scale=1,width=3.4in]{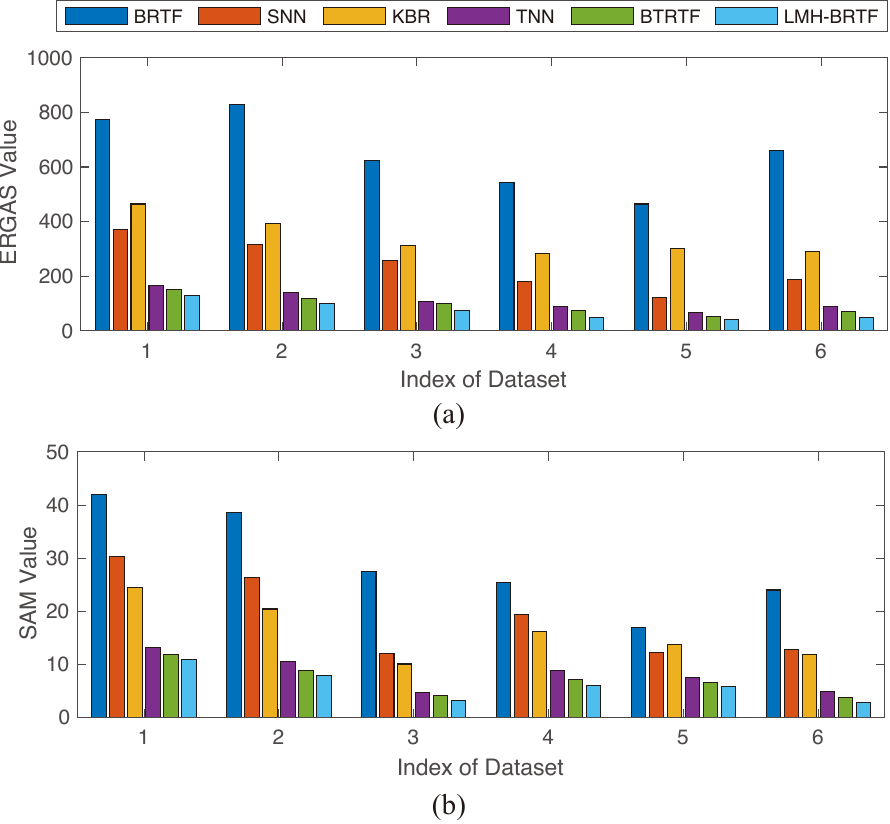}
	\caption{Comparison of light-field images denoising results on 6 datasets. (a) Comparison with respect to ERGAS. (b) Comparison with respect to SAM.}
	\label{exp3_values}
\end{figure}

\begin{figure*}[htbp]
	\centering
	\includegraphics[scale=1,width=7in]{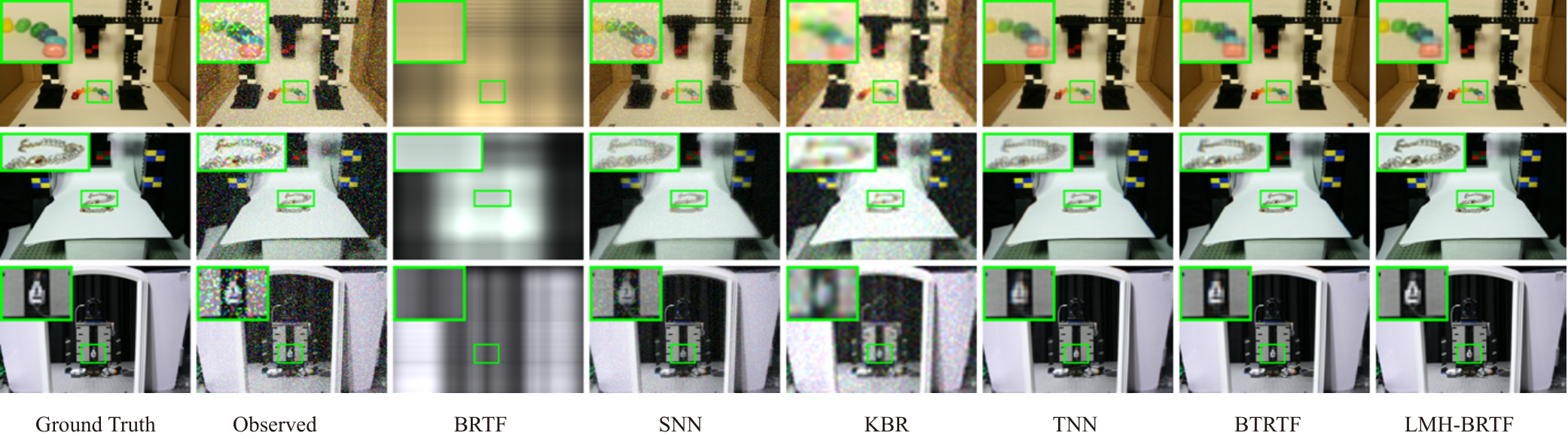}
	\caption{Light-field images denoising results on three representative datasets with $\rho=20\%$ and $\sigma^2 = 10^{-4}$. The names of the datasets from top to bottom are Beans, Bracelet, and Portrait, respectively.}
	\label{exp3_show}
\end{figure*}

Similar to the evaluation metrics in section IV.B, we also use the PSNR, SSIM\cite{SSIM}, FSIM\cite{FSIM}, ERGAS\cite{ERGAS}, SAM\cite{SAM}, and CPU time. We also compare LMH-BRTF with the TRPCA methods listed in table \ref{tab:competing methods} and set the invertible linear transform $L$ as the DFT transformation whose corresponding constant $\varphi = I_3 \times \cdots I_d$ for LMH-BRTF.

All the parameters in the competing methods are optimally tuned or set as the suggested values in the corresponding papers. Specifically, for BRTF, we set the initial rank to $\max(I_1,I_2,I_3,I_4,I_5)$; For SNN, we set the parameter $\boldsymbol{\lambda}=[\lambda_1,\lambda_2,\lambda_3,\lambda_4,\lambda_5]$ as $[15,15,1.5,400,400]$; For KBR, we set $\beta = 3000$, $\gamma = 100\beta$, $\lambda=10$, $\mu = 250$, $\rho = 1.05$; For TNN, we set the balancing parameter $\lambda=1 / \sqrt{I_{3}\times I_4 \times I_5 \times \max \left(I_{1}, I_{2}\right)}$; For BTRTF, we set the initial multi-rank to $100 \times \operatorname{ones}(I_3 \times I_4 \times I_5,1)$, and the sparse variance $\sigma_{0}^{2} = 10^{-7}$; For LMH-BRTF, we initialize the multi-rank to $100 \times \operatorname{ones}(I_3 \times I_4 \times I_5,1)$. Similar as above, we set the convergence criterion to $tol < 10^{-4}$ and the sparse variance $\sigma_{0}^{2}= 10^{-7}$.

The denoising results on six light-field images datasets are shown in table \ref{tab:tableExp3}. From these results, we have the following observations. Firstly, we can see that LMH-BRTF consistently achieves better light-field images denoising performance than the order-3 t-SVD based methods, i.e., TNN and BTRTF. This is because those order-3 t-SVD based methods should first reshape the higher-order tensors into order-3 tensors, and thus can not effectively utilize the information among dimensions. The higher-order method, LMH-BRTF, instead takes advantage of the multi-dimensional structure of data. Secondly, the denoising performance of LMH-BRTF is better than CP decomposition and Tucker decomposition based methods, i.e., BRTF, SNN and KBR. This is because those classical CP decomposition and Tucker decomposition based TRPCA methods generally lack enough ﬂexibility in characterizing tensors with complex low-rank structures. In addition, the running time of LMH-BRTF is comparable to that of TNN and BTRTF. This is because although the order-$d$ t-SVD framework performs more complex operations when dealing with higher-order tensors compared to oder-3 t-SVD, these operations can be implemented efficiently and do not impose excessive computational burdens. Moreover, BRTF is the most efficient method according to the CPU time, but it does not perform well. We also give the denoising results of the six light-field images datasets with regard to ERGAS and SAM values obtained by all algorithms in Fig. \ref{exp3_values} and these results consistently show that LMH-BRTF can achieve better light-field images denoising performance compared to other algorithms. Fig. \ref{exp3_show} further shows the visualized denoising results of each algorithm from Beans, Bracelet, and Portrait datasets, and it turns out that LMH-BRTF can obtain finer and more accurate denoising results compared to other algorithms. For example, in the Beans dataset, LMH-BRTF can obtain clearer edges of beans compared to other algorithms, especially the edge between the green and dark blue beans. In the Bracelet dataset, the details of the bracelet obtained by LMH-BRTF are more accurate and richer. In the Portrait dataset, LMH-BRTF obtains results without blurring and ghosting.

\section{Conclusion}

In this article, a low-multi-rank high-order Bayesian robust tensor factorization method LMH-BRTF is proposed for solving the order-$d$ TRPCA problem. In this method, we assume that the observed corrupted tensor can be decomposed into three parts, i.e., the low-rank component, the sparse component, and the noise component. We construct a low-rank model for the low-rank component based on order-$d$ t-SVD and impose a proper prior for the model, so that LMH-BRTF can automatically determine the multi-rank of an order-$d$ tensor. In the meanwhile, we explicitly model the noises besides the sparse noise, so that LMH-BRTF can leverage information from the noises, which leads to improved performance of TRPCA. Comprehensive synthetic and real-world data experiments and comparisons demonstrate the superiority of LMH-BRTF.


%

\ifCLASSOPTIONcaptionsoff
  \newpage
\fi



\bibliography{IEEEfull}

\begin{thebibliography}{1}

  \bibitem{1Statis}
  P. McCullagh, \emph{Tensor methods in statistics}. Courier Dover Publications, 2018.

  \bibitem{2SignalProces}
  N. D. Sidiropoulos, L. De Lathauwer, X. Fu, \emph{et al.} ``Tensor decomposition for signal processing and machine learning,'' \emph{IEEE Trans. Signal Process.}, vol. 65, no. 13, pp. 3551-3582, Jul. 2017.

  \bibitem{3CV}
  Y. Panagakis, J. Kossaiﬁ, G. G. Chrysos, \emph{et al.} ``Tensor methods in computer vision and deep learning,'' \emph{Proc. IEEE}, vol. 109, no. 5, pp. 863-890, May 2021.


  \bibitem{4MachinLearn}
  M. Signoretto, Q. Tran Dinh, L. De Lathauwer, \emph{et al.} ``Learning with tensors: a framework based on convex optimization and spectral regularization,'' \emph{Mach. Learn.}, vol. 94, no. 3, pp. 303-351, Mar. 2014.


  \bibitem{5CPrank}
  T. G. Kolda and B. W. Bader, ``Tensor decompositions and applications,'' \emph{SIAM Rev.}, vol. 51, no. 3, pp. 455-500, Aug. 2009.

  \bibitem{CPrankNP1}
  J. M. Landsberg, \emph{Tensors: Geometry and Applications}. American Mathematical Society, 2012.

  \bibitem{CPrankNP2}
  C. J. Hillar and L.-H. Lim, ``Most tensor problems are NP-hard,'' \emph{J.ACM}, vol. 60, no. 6, pp. 1-39, Nov. 2013.

  \bibitem{CPrankNP3}
  Q. Zhao, L. Zhang, and A. Cichocki, ``Bayesian CP factorization of incomplete tensors with automatic rank determination,'' \emph{IEEE Trans. Pattern Anal. Mach. Intell.}, vol. 37, no. 9, pp. 1751-1763, Sep. 2015.


  \bibitem{BRTF}
  Q. Zhao, G. Zhou, L. Zhang, \emph{et al.} ``Bayesian robust tensor factorization for incomplete multiway data,'' \emph{IEEE Trans. Neural Netw. Learn. Syst.}, vol. 27, no. 4, pp. 736-748, Apr. 2016.

  \bibitem{CPdecomp2}
  X. Chen, Z. Han, Y. Wang, \emph{et al.} ``A generalized model for robust tensor factorization with noise modeling by mixture of Gaussians,'' \emph{IEEE Trans. Neural Netw. Learn. Syst.}, vol. 29, no. 11, pp. 5380-5393, Nov. 2018.

  \bibitem{Tuckerrank1}
  L. R. Tucker, ``Some mathematical notes on three-mode factor analysis,'' \emph{Psychometrika}, vol. 31, no. 3, pp. 279-311, Sep. 1966.

  \bibitem{Tuckerrank2}
  X. Li, Y. Ye, and X. Xu, ``Low-rank tensor completion with total variation for visual data inpainting,'' in \emph{Proc. AAAI Conf. Artif. Intell.}, San Francisco, USA, Feb. 2017, pp. 1-7.

  \bibitem{SNN}
  B. Huang, C. Mu, D. Goldfarb, \emph{et al.} ``Provable models for robust low-rank tensor completion,'' \emph{Pacific J. Optim.}, vol. 11, no. 2, pp. 339-364, 2015.

  \bibitem{Tuckerconvex}
  C. Mu, B. Huang, J. Wright, \emph{et al.} ``Square deal: Lower bounds and improved relaxations for tensor recovery,'' in \emph{Proc. Int. Conf. Mach. Learn.}, Beijing, China, Jun. 2014, pp. 73-81.

  \bibitem{TT1}
  I. V. Oseledets, ``Tensor-train decomposition,'' \emph{SIAM J. Sci. Comput.}, vol. 33, no. 5, pp. 2295-2317, Jan. 2011.
  \bibitem{TT2}
  J. A. Bengua, H. N. Phien, H. D. Tuan, \emph{et al.} ``Efficient tensor completion for color image and video recovery: Low-rank tensor train,'' \emph{IEEE Trans. Image Process.}, vol. 26, no. 5, pp. 2466-2479, May 2017.

  \bibitem{TR1}
  W. Wang, V. Aggarwal, and S. Aeron, ``Efficient low rank tensor ring completion,'' in \emph{Proc. IEEE Int. Conf. Comput. Vis.}, Venice, Italy, Oct. 2017, pp. 5698-5706.

  \bibitem{TR2}
  Z. Long, C. Zhu, J. Liu, \emph{et al.} ``Bayesian low rank tensor ring for image recovery,'' \emph{IEEE Trans. Image Process.}, vol. 30, pp. 3568-3580, Mar. 2021.


  \bibitem{tSVD1}
  M. E. Kilmer, K. Braman, N. Hao, \emph{et al.} ``Third-order tensors as operators on matrices: A theoretical and computational framework with applications in imaging,'' \emph{SIAM J. Matrix Anal. Appl.}, vol. 34, no. 1, pp. 148-172, Jan. 2013.

  \bibitem{tSVD2}
  M. E. Kilmer and C. D. Martin, ``Factorization strategies for third-order tensors,'' \emph{Linear Algebra Appl.}, vol. 435, no. 3, pp. 641-658, Aug. 2011.

 
  \bibitem{tsvdAdvan1}
  Z. Zhang, G. Ely, S. Aeron, \emph{et al.} ``Novel methods for multilinear data completion and de-noising based on tensor-SVD,'' in \emph{Proc. IEEE Conf. Comput. Vis. Pattern Recognit.}, Columbus, USA, Jun. 2014, pp. 3842-3849.

  \bibitem{tsvdAdvan2}
  J. Lou and Y.-M. Cheung, ``Robust low-rank tensor minimization via a new tensor spectral $k$-support norm,'' \emph{IEEE Trans. Image Process.}, vol. 29, pp. 2314-2327, Oct. 2019.

  \bibitem{tsvdAdvan3}
  A. Wang, Z. Jin, and G. Tang, ``Robust tensor decomposition via t-SVD: Near-optimal statistical guarantee and scalable algorithms,'' \emph{Signal Process.}, vol. 167, no. 107319, Oct. 2020.

  \bibitem{tsvdAdvan4}
  Y. Chen, S. Wang, C. Peng, \emph{et al.} ``Generalized nonconvex low-rank tensor approximation for multi-view subspace clustering,'' \emph{IEEE Trans. Image Process.}, vol. 30, pp. 4022-4035, Mar. 2021.

  \bibitem{TNN}
  C. Lu, J. Feng, Y. Chen, \emph{et al.} ``Tensor robust principal component analysis with a new tensor nuclear norm,'' \emph{IEEE Trans. Pattern Anal. Mach. Intell.}, vol. 42, no. 4, pp. 925-938, Jan. 2019.

  \bibitem{BTRTF}
  Y. Zhou and Y.-M. Cheung, ``Bayesian low-tubal-rank robust tensor factorization with multi-rank determination,'' \emph{IEEE Trans. Pattern Anal. Mach. Intell.}, vol. 43, no. 1, pp. 62-76, Jan. 2021.


  \bibitem{tSVDmethod1}
  J. Wang, J. Hou, and Y. C. Eldar, ``Tensor robust principal component analysis from multilevel quantized observations,'' \emph{IEEE Trans. Inf. Theory}, vol. 69, no. 1, pp. 383-406, Aug. 2022.


  \bibitem{tSVDmethod2}
  C. Lu, J. Feng, Y. Chen, \emph{et al.} ``Tensor robust principal component analysis: Exact recovery of corrupted low-rank tensors via convex optimization,'' in \emph{Proc. IEEE Conf. Comput. Vis. Pattern Recognit.}, Las Vegas, USA, 2016, pp. 5249-5257.


  \bibitem{tSVDmethod3}
  Q. Gao, P. Zhang, W. Xia, \emph{et al.} ``Enhanced tensor RPCA and its application,'' \emph{IEEE Trans. Pattern Anal. Mach. Intell.}, vol. 43, no. 6, pp. 2133-2140, Jun. 2021.


  \bibitem{HTNN}
  W. Qin, H. Wang, F. Zhang, \emph{et al.} ``Low-rank high-order tensor completion with applications in visual data,'' \emph{IEEE Trans. Image Process.}, vol. 31, pp. 2433-2448, Mar. 2022. 

  \bibitem{orderDtSVD1}
  C. D. Martin, R. Shafer, and B. LaRue, ``An order-$p$ tensor factorization with applications in imaging,'' \emph{SIAM J. Sci. Comput.}, vol. 35, no. 1, pp. A474-A490, Jan. 2013.

  \bibitem{orderDtSVD2}
  A. Bibi and B. Ghanem, ``High order tensor formulation for convolutional sparse coding,'' in \emph{Proc. IEEE Int. Conf. Comput. Vis.}, Venice, Italy, Oct. 2017, pp. 1772-1780.

  \bibitem{orderDtSVD3}
  M. E. Kilmer, L. Horesh, H. Avron, \emph{et al.} ``Tensortensor algebra for optimal representation and compression of multiway data,'' \emph{Proc. Natl. Acad. Sci.}, vol. 118, no. 28, Jul. 2021.
  
  \bibitem{heuristic MR deter}
  P. Zhou, C. Lu, Z. Lin, \emph{et al.} ``Tensor factorization for low-rank tensor completion,'' \emph{IEEE Trans. Image Process.}, vol. 27, no. 3, pp. 1152-1163, Mar. 2018.

  \bibitem{ARD}
  R. M. Neal, \emph{Bayesian Learning for Neural Networks}. Springer, 2012.


  \bibitem{VB1}
  M. J. Beal, \emph{Variational algorithms for approximate Bayesian inference}. Univ. of London, 2003.

  \bibitem{VB2}
  J. Winn, C. M. Bishop, and T. Jaakkola, ``Variational message passing,'' \emph{J. Mach. Learn. Res.}, vol. 6, no. 4, pp. 1-33, Apr. 2005.

  \bibitem{SSIM}
	Z. Wang, A. C. Bovik, H. R. Sheikh, \emph{et al.} ``Image quality assessment: From error visibility to structural similarity,'' \emph{IEEE Trans. Image Process.}, vol. 13, no. 4, pp. 600-612, Apr. 2004.

  \bibitem{FSIM}
  L. Zhang, L. Zhang, X. Mou, \emph{et al.} ``FSIM: A feature similarity index for image quality assessment,'' \emph{IEEE Trans. Image Process.}, vol. 20, no. 8, pp. 2378-2386, Aug. 2011.

  \bibitem{ERGAS}
  L. Wald, \emph{Data Fusion: Deﬁnitions and Architectures: Fusion of Images of Different Spatial Resolutions}. Paris, France: Presses des l'Ecole MINES, 2002.

  \bibitem{SAM}
  F. A. Kruse, A. B. Lefkoff, J. W. Boardman, \emph{et al.} ``The spectral image processing system (SIPS) - Interactive visualization and analysis of imaging spectrometer data,'' \emph{Remote Sens. Environ.}, vol. 44, no.2-3, pp.145-163, May 1993.


  \bibitem{KBR}
  Q. Xie, Q. Zhao, D. Meng, \emph{et al.} ``Kronecker-basis-representation based tensor sparsity and its applications to tensor recovery,'' \emph{IEEE Trans. Pattern Anal. Mach. Intell.}, vol. 40, no. 8, pp. 1888-1902, Aug. 2017.



\end{thebibliography}

%








\end{document}